\documentclass[letterpaper,preprint,12pt]{elsarticle}
\makeatletter
\def\ps@pprintTitle{%
 \let\@oddhead\@empty
 \let\@evenhead\@empty
 \def\@oddfoot{}%
 \let\@evenfoot\@oddfoot}
\makeatother

\usepackage[left=0.75in, right=0.75in,top = 0.75in, bottom = 1in]{geometry}

\usepackage{bm}
\usepackage{dsfont}
\usepackage{amssymb}

\usepackage{mathtools}

\usepackage[utf8]{inputenc}
\usepackage[english]{babel}

\usepackage{algpseudocode}
\usepackage{algorithm}
\usepackage[capitalize]{cleveref}
\usepackage{amsthm}
\usepackage{graphicx}
\usepackage{amsmath}
\usepackage[version=4]{mhchem}
\usepackage{siunitx}
\usepackage{longtable,tabularx}
\usepackage{booktabs}

\newtheoremstyle{mystyle}
  {}
  {}
  {\itshape}
  {}
  {\bfseries}
  {.}
  { }
  {\thmname{#1}\thmnumber{ #2}\thmnote{ (#3)}}

\theoremstyle{mystyle}
\newtheorem{theorem}{Theorem}[section]

\newtheorem{cor}{Corollary}
\newtheorem{lem}{Lemma}
\newtheorem{definition}{Definition}[section]

\usepackage{subfig}

\renewcommand{\epsilon}{{\varepsilon}}

\DeclareMathOperator*{\argmin}{argmin}

\usepackage{centernot}
\usepackage{mathtools}
\usepackage{stmaryrd}

\makeatletter
\newcommand{\xMapsto}[2][]{\ext@arrow 0599{\Mapstofill@}{#1}{#2}}
\def\Mapstofill@{\arrowfill@{\Mapstochar\Relbar}\Relbar\Rightarrow}
\makeatother

\usepackage{lineno}

\title{New Approaches to Inverse Structural Modification Theory using Random Projections}

\begin{document}

\begin{frontmatter}

\author{Prasad Cheema$^1$, Mehrisadat M. Alamdari$^2$, Gareth A. Vio$^1$}

\address{$^1$ School of AMME, The University of Sydney, NSW 2006, Australia}
\address{$^2$ School of Civil and Environmental Engineering, University of New South Wales, Sydney, NSW 2052, Australia}
\pagenumbering{gobble}

\begin{abstract}
In many contexts the modal properties of a structure change, either due to the impact of a changing environment, fatigue, or due to the presence of structural damage. For example during flight, an aircraft’s modal properties are known to change with both altitude and velocity. It is thus important to quantify these changes given only a truncated set of modal data, which is usually the case experimentally. This procedure is formally known as the generalised inverse eigenvalue problem. In this paper we experimentally show that first-order gradient-based methods that optimise objective functions defined over a modal are prohibitive due to the required small step sizes. This in turn leads to the justification of using a non-gradient, black box optimiser in the form of particle swarm optimisation. We further show how it is possible to solve such inverse eigenvalue problems in a lower dimensional space by the use of random projections, which in many cases reduces the total dimensionality of the optimisation problem by 80\% to 99\%. Two example problems are explored involving a ten-dimensional mass-stiffness toy problem, and a one-dimensional finite element mass-stiffness approximation for a Boeing 737-300 aircraft.
\end{abstract}


\end{frontmatter}

\section{Introduction}
\label{S:1}

Eigenvalue problems are common in the engineering context \cite{bathe1973solution,elishakoff1991some}. As such, they have been used in a plethora of applications such as in analysing the state matrix of an electronic power system \cite{lima2000assessing}, in studying the  aeroelastic instability for wind turbines \cite{hansen2007aeroelastic}, for determining the spectral radius of Jacobian matrices \cite{day1984run}, and in the operational modal analysis of a structures \cite{sun2017automated}. The most common eigenvalue problem, known as the the \textit{direct} or \textit{forward} problem involves determining the impact of a known set of modifications to a group of matrices, either by computing the eigenvalues, eigenvectors, singular values, or singular vectors of the group of matrices. The direct problem is well studied, and is the subject of many elementary courses in linear algebra, but the \textit{inverse} problem is much more complex. 

The inverse problem tries to find or infer a particular type of modification which was applied to a set of matrices, from a larger set of possible modifications, using mainly spectral information \cite{chu2005inverse}. It is clear that this problem would be trivial if all the spectral information of the  system before and after any modifications were known (that is, we are not dealing with a truncated modal system), or if the desired modifications were completely unstructured (that is, they are allowed to be any value). Thus in order to strive for more physical, and mathematical solutions we often try to restrict the group of possible matrices for the inverse eigenvalue problem. In a recent review article, Chu \cite{chu1998inverse} devised a collection of thirty-nine possible inverse eigenvalue problems. These problems were roughly categorized based on their: paramterisation, underlying structure, and the partiality of the system description (that is, whether or not we have complete modal information). A summary of the most common kinds of inverse eigenvalue problems are given in Figure \ref{fig:inverse_eig_value_probs}, where the following terminology is used:

\begin{itemize}
    \item MVIEP: Multivariate inverse eigenvalue problem
    \item LSIEP: Least square inverse eigenvalue problem
    \item PIEP: Parameterised inverse eigenvalue problem
    \item SIEP: Structured inverse eigenvalue problem
    \item PDIEP: Partially described inverse eigenvalue problem
    \item AIEP: Additive inverse eigenvalue problem
    \item MIEP: Multiplicative inverse eigenvalue problem
\end{itemize}

In this paper we aim to explore AIEPs which have a highly general parameterisation, in order to demonstrate the potential capabilities of random projections for the inverse eigenvalue problem. 

\begin{figure}[ht!]
\centering
  \includegraphics[width=0.45\linewidth]{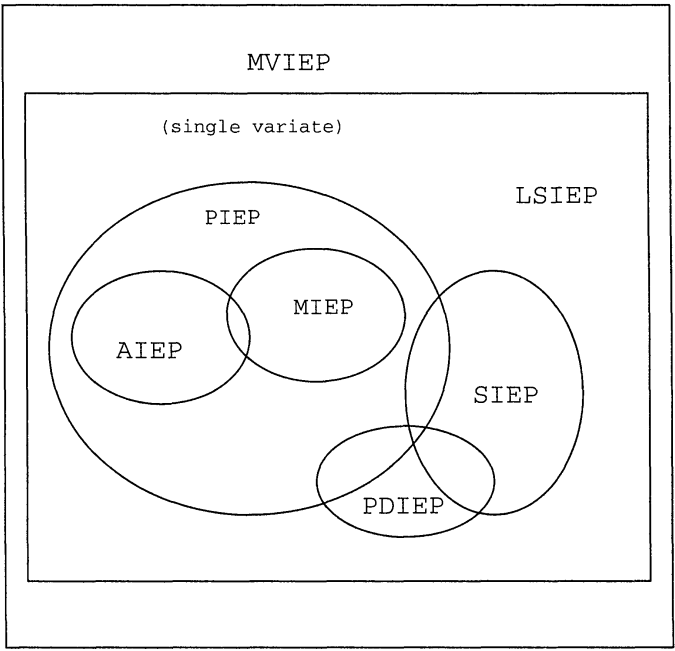}
  \caption{An overview of some of the general classes of inerse eigenvalue problems as defined by Chu \cite{chu1998inverse}.}
  \label{fig:inverse_eig_value_probs}
\end{figure}

In order to solve for the AIEP we shall define an optimisation problem. Although there are many optimisation procedures available for solving such problems, a particle swarm optimiser (PSO) is used in this paper. PSO is an optimisation procedure first introduced by Eberhart and Kennedy in 1995 \cite{eberhart1995new}. It is a stochastic, population-based optimisation procedure modeled on the observed behaviour of animals which exhibit swarm-like tendencies, such in the social behaviour of birds or insects. Because of this PSO, tries to mimic swarm-like behaviour with each \textit{particle} having access to both: a personal best solution, \textit{and} access to the global optimum, thereby enabling the sharing of information across the swarm. This introduces the idea of the classic \textit{exploration-exploitation} trade-off since the particles are either allowed to \textit{exploit} the current global optimum, or \textit{explore} further if their local optimum is far away from the current global \cite{eberhart1995new,van2010convergence}. 

PSO is used in this paper to perform the optimisation for two main reasons. Firstly, to the knowledge of the authors, it has not previously been used in the context of this problem (AIEP), hence there is novelty in doing so. Secondly, it is a known black-box, gradient-free optimiser which makes it simple to work with since there is no requirement to compute the Jacobian, or calculate analytical gradients. \cite{van2010convergence} 
Moreover as we shall demonstrate experimentally, gradient step sizes are required to be very small in the inverse eigenvalue problem, with the issue exacerbating in higher dimensions. PSO as an algorithm has been successfully used in many different areas, including but not limited to reactive power and voltage control problems \cite{yoshida2000particle}, in the study of material degradation for aeroelastic composites \cite{vishwanathan2017multi}, composites structures with robustness \cite{vishwanathan2017robust}, and in the optimum design of Proportional-Integral-Derivative (PID) control \cite{gaing2004particle} which additionally helps in justifying its potential to work well in this context. 

Regardless, of the choice of optimiser, all optimisation procedures are known to suffer from the \textit{curse of dimensionality}, and since structural problems are generally able to grow without bound in terms of degrees of freedom (for example, a finite element model can keep growing in the number of elements), it is important to devise methods which can help reduce, or at least limit the rate of growth of these dimensions. Random projections have recently emerged as a powerful method to address the problem of dimensionality reduction \cite{bingham2001random}. This is because theory (in particular the Johnson–Lindenstrauss lemma) suggests that certain classes of random matrices are able to preserve Euclidean distances to within a tolerable error bound in the lower dimensional space \cite{johnson1984extensions}. As a result, random projections have been used successfully to reduce the dimensionality of the underlying optimisation problem, consequently allowing for optimisation to be performed in this lower dimensional space \cite{wang2016bayesian,gardner2014bayesian,krummenacher2014radagrad}. 

Thus, it is ultimately the aim of this paper to explore the impact of random projections and how they may be used in connection with the PSO algorithm for the generalised inverse eigenvalue problem of an additive nature. In the following sections we demonstrate experimentally that gradient-based approaches lose accuracy in even moderate step sizes, and clarify the theory that we shall use from random projections to help in lowering the dimensionality of the underlying optimisation problem. Lastly we showcase a gamut of positive results on a 10 dimensional (meaning a matrix of size 10) toy problem, and 1 dimensional finite element model based on aircraft data for the Boeing 737-300 aircraft.

\section{Background and Methodology}

\subsection{The Generalised Eigenvalue Problem}

Ultimately it is the aim of this paper to use particle swarm optimisation (PSO) in order to try and solve the generalised \textit{inverse} eigenvalue problem, with the use of random projections. We thus commence by formalising the notion of the generalised eigenvalue problem here. 

Suppose we have the generalised eigenvalue problem as shown in Equation \ref{eqn:gen_eig_prob}, which represents an undamped mechanical vibration system.

\begin{equation} \label{eqn:gen_eig_prob}
    \mathbf{M}\Ddot{\mathbf{x}} + \mathbf{Kx} = 0
\end{equation}

where  $\mathbf{M},\mathbf{K} \in \mathbb{R}^{N \times N}$, and $\mathbf{x}\in\mathbb{R}^N$. The eigensystem for Equation \ref{eqn:gen_eig_prob} defines the following set of eigenvalue, eigenvector pairs: $\mathbb{A} = \{(\lambda_i,\mathbf{v}_i)|i = 1,..,N; \mathbf{v}_i \in \mathbb{R}^N, \lambda \in \mathbb{R}\}$. Furthermore we assume that the eigenvalue problem is perturbed via the addition of some arbitrary matrices we denote as $\mathbf{\Delta}\in (\mathbb{R}^{N\times N}, \mathbb{R}^{N\times N}$). That is, from here onwards whenever the $\mathbf{\Delta}$ symbol is written in isolation, in a bold-type font it denotes a 2-tuple of perturbation matrices of the system, that is, $\mathbf{\Delta} \coloneqq (\Delta \mathbf{M}, \Delta \mathbf{K})$, where $\Delta \mathbf{M} \in \mathbb{R}^{N\times N}$ and $\Delta \mathbf{K} \in \mathbb{R}^{N\times N}$.

\begin{equation} \label{eqn:mod_eig_prob}
( \mathbf{M} + \Delta \mathbf{M})\Ddot{\mathbf{x}} + (\mathbf{K}+\Delta \mathbf{K})\mathbf{x} = 0    
\end{equation}

The eigenpairs for the modified system shown in Equation \ref{eqn:mod_eig_prob} may be represented via the following set of eigenpairs: $\mathbb{B} = \{(\sigma_i,\mathbf{w}_i)|i = 1,..,N; \mathbf{w}_i \in \mathbb{R^N}, \sigma \in \mathbb{R}\}$. However, if both the initial and modified systems are full rank systems, then it would be trivial to obtain the $\Delta \mathbf{M}$ and $\Delta \mathbf{K}$ matrices. Thus for this paper we assume that we only have access to a truncated eigensystem for the modified system. That is, we only have access to some subset of the pairs: $\mathbb{C} \subset \mathbb{B}$, where $|\mathbb{C}| = n < N$. Thus our objective function in the search for the optimal $\mathbf{\Delta}$ matrices reflects this, in Equation \ref{eqn:obj_fun_init}.

\begin{equation} \label{eqn:obj_fun_init}
\Delta \mathbf{M}^{\star}, \Delta \mathbf{K}^{\star} = \argmin(||\bm{\sigma}^{\dagger}_{1:n} - \bm{\sigma}_{1:n}(\bm{\Delta})||_2^2),
\end{equation}

where $||\cdot ||_2^2 $ denotes the square of the standard 2-norm,  $\bm{\sigma}^{\dagger}$ denotes the desired eigenvalues, and $\bm{\sigma}$ refers to the eigenvalues as calculated from applying the $\bm{\Delta}$ matrices (clarified in the prior paragraph in reference to set $\mathbb{B}$). As is made clear in Equation \ref{eqn:obj_fun_init}, we only consider the first $n < N$ dimensions, since we are dealing with a reduced set of eigenvalues. 

In this paper we propose investigating the solutions for the objective function shown in Equation \ref{eqn:obj_fun_init} via PSO. We aim to use a non-gradient based, black-box optimisation since a first order perturbation analysis of the modified eigenvalues seem to suggest that for higher dimensional problems the step-sized used by gradient-based approaches may become prohibitively small. We establish this idea by first developing Lemma \ref{lem:gen_perturb} as follows. 

\begin{lem} \label{lem:gen_perturb}
Suppose we have the two following generalised eigenvalue problems,
\begin{align}
    \lambda \mathbf{M} \mathbf{v} &= \mathbf{K} \mathbf{v} \label{eqn:not_pertubed}\\
    (\lambda+\delta \lambda) (\mathbf{M} +\Delta \mathbf{M}) (\mathbf{v}+\delta \mathbf{v}) &= (\mathbf{K} + \Delta \mathbf{K}) (\mathbf{v} + \delta \mathbf{v}) \label{eqn:pertubed},
\end{align}
where $\bm{\Delta}$ perturbations are controlled system inputs, and the $\delta$ perturbations are a consequence of applying $\bm{\Delta}$. Then,
\begin{align} \label{eqn:pert_gen}
    \delta \lambda_i = \frac{\mathbf{v}_i^{\intercal}\left(\Delta \mathbf{K} - \lambda_i \Delta \mathbf{M}\right)\mathbf{v}_i}{\mathbf{v}_i^{\intercal}\Delta \mathbf{M} \mathbf{v}_i},
\end{align}
if $\mathbf{M}$ and $\mathbf{K}$ are Hermitian.
\end{lem}

\begin{proof}
By expanding Equation \ref{eqn:pertubed}, removing higher order terms (that is, keeping only linear terms), and considering the $i^{\text{th}}$ eigenvalue-eigenvector pairs we arrive at,
\begin{align} \label{eqn:proof_one}
  \mathbf{K}\delta \mathbf{v}_i + \Delta \mathbf{K} \mathbf{v}_i = \lambda_i \mathbf{M} \delta \mathbf{v}_i + \delta \lambda_i \mathbf{M} \mathbf{v}_i + \lambda_i \Delta \mathbf{M} \mathbf{v}_i.
\end{align}

Since $\mathbf{M}$ and $\mathbf{K}$ are Hermitian it implies that that the eigenvectors of Equation \ref{eqn:not_pertubed} are mutually $\mathbf{M}$-orthogonal. Moreover since they are assumed diagonsliable, these eigenvectors form a complete basis. Hence we can express each perturbation vector, $\delta \mathbf{v}_i$ as a sum of the eigenvectors of Equation \ref{eqn:not_pertubed}. As an equation this is,

\begin{align} \label{eqn:eig_basis}
\delta \mathbf{v}_i = \sum_{k=1}^N c_k \mathbf{v}_k,
\end{align}

for some arbitrary constants $c_k \in \mathbb{R}$. Thus, substituting Equation \ref{eqn:eig_basis} into Equation \ref{eqn:proof_one}, and using Equation \ref{eqn:not_pertubed}:

\begin{align} \label{eqn:proof_two}
    \sum_{k=1}^N c_k \lambda_k \mathbf{M} \mathbf{v}_k + \Delta \mathbf{K} \mathbf{v}_i = \lambda_i \mathbf{M}  \sum_{k=1}^N c_k \mathbf{v}_k + \delta \lambda_i \mathbf{M} \mathbf{v}_i + \lambda_i \Delta \mathbf{M} \mathbf{v}_i.  
\end{align}

Finally, left multiplying Equation \ref{eqn:proof_two} by $\mathbf{v}_i^{\intercal}$, and re-arranging results in,

\begin{align*} 
     \delta \lambda_i = \frac{\mathbf{v}_i^{\intercal}\left(\Delta \mathbf{K} - \lambda_i \Delta \mathbf{M}\right)\mathbf{v}_i}{\mathbf{v}_i^{\intercal}\Delta \mathbf{M} \mathbf{v}_i},
\end{align*}

since the eigenvectors $\mathbf{v}_i$ are $\mathbf{M}$-orthogonal.

\end{proof}

\begin{cor} \label{cor:cor_perturb}
Suppose we have the two following standard eigenvalue problems,
\begin{align}
    \lambda \mathbf{v} &= \mathbf{K} \mathbf{v} \label{eqn:not_pertubed_2}\\
    (\lambda+\delta \lambda) (v+\delta v) &= (\mathbf{K} + \Delta \mathbf{K}) (\mathbf{v} + \delta \mathbf{v}) \label{eqn:pertubed_2}.
\end{align}
Then,
\begin{align} \label{eqn:pert_K}
    \delta \lambda_i &= \frac{\mathbf{v}_i^{\intercal}\Delta \mathbf{K} \mathbf{v}_i}{\mathbf{v}_i^{\intercal} \mathbf{v}_i},
\end{align}
if $\mathbf{K}$ is Hermitian.
\end{cor}

\begin{proof}
The proof follows similarly from repeating the steps shown in Lemma \ref{lem:gen_perturb}, in the absences of the $\mathbf{M}$ and $\Delta \mathbf{M}$ matrices.
\end{proof}

Note that in Equations \ref{eqn:pert_gen} and \ref{eqn:pert_K}, it is not necessarily true that $\mathbf{v}_i^{\intercal}\Delta \mathbf{M} \mathbf{v}_i = 1$ since each eigenvector $\mathbf{v}_i$ is only orthogonal with respect to the $\bm{M}$ matrix, and not $\Delta \bm{M}$. A similar argument may be made with the $\bm{K}$ and $\Delta \mathbf{K}$ matrices. In addition, Equations \ref{eqn:pert_gen} and \ref{eqn:pert_K} make clear how the perturbation matrices, $\bm{\Delta}$, impact the changes in the eigenvalues, $\delta \lambda$, up to a linear approximation (since in the derivation the higher order effects were ignored). As a result this relationship may be used in better understanding gradient-based relationships for eigenvalue problems. That is, we may use these to analyse the potential accuracy and or quality of gradient-based approaches for such problems. An investigation of these ideas is made clear in Figure \ref{fig:step_sizes}.

\begin{figure}[ht] 
\centering
\subfloat[Different $p$ magnitudes plotted against $|\Delta \delta \lambda_1|_{\mu}$ term as a percentage error.]{
\hspace{-1cm}
\includegraphics[width=0.55\linewidth]{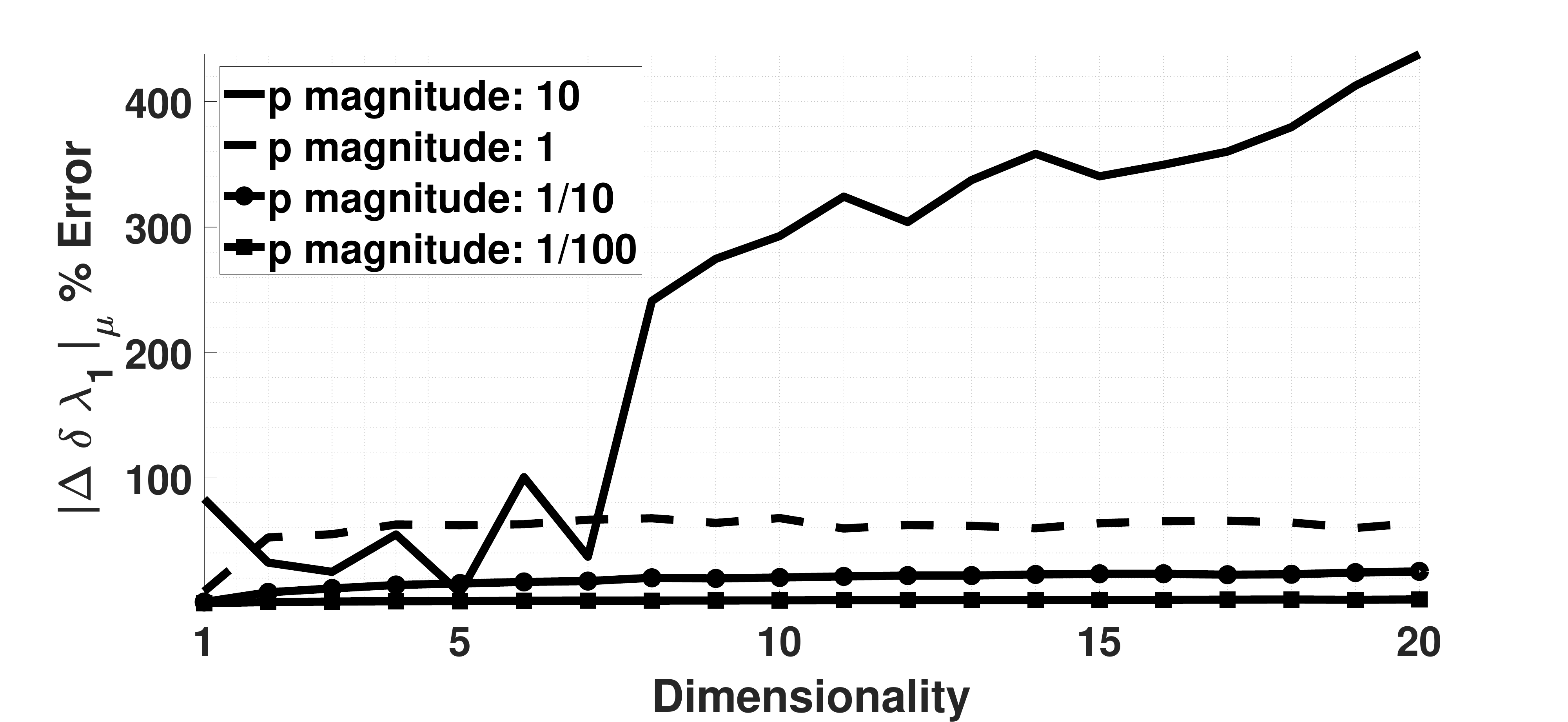}
\label{fig:step_sizes_not_zoomed}
}
\subfloat[A zoomed-in version of subplot \ref{fig:step_sizes_not_zoomed} to more clearly show the cases of $p=1/10$ and $p=1/100$.]{
\label{fig:step_sizes_zoomed}
\hspace{-1cm}
      \includegraphics[width=0.55\linewidth]{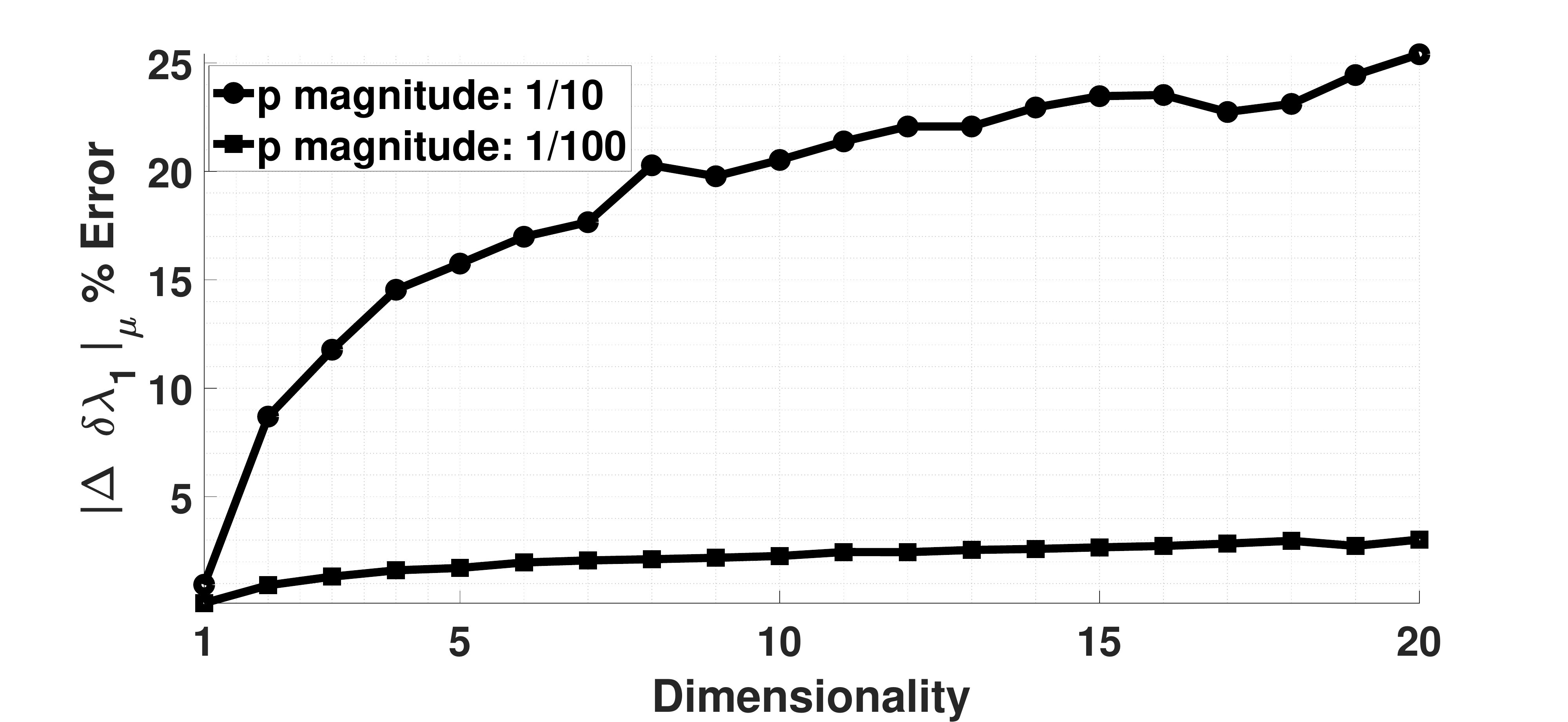}
}
  \caption{Plots describing how the average magnitudes of the elements (defined by Equation \ref{eqn:p_dist}) inside the $\Delta \bm{M}$ and $\Delta \bm{K}$ matrices affect the $\Delta \delta \lambda_1$ term, defined in Equation \ref{eqn:perc_err}.}
    \label{fig:step_sizes}
\end{figure}

Equation \ref{eqn:perc_err} is used order to calculate the percentage errors for the $|\Delta \delta \lambda_1|$ values in Figure \ref{fig:step_sizes}, where $\delta \lambda_1^{\dagger}$ refers to the change in first eigenvalue as obtained from Equation \ref{eqn:pert_gen}, and $\delta \lambda_1$ is obtained using Equations \ref{eqn:gen_eig_prob} and \ref{eqn:mod_eig_prob}, which refers to the benchmark correct value.

\begin{align} \label{eqn:perc_err}
    |\Delta \delta \lambda_1|_{\mu} &\coloneqq \mathbb{E}(|\Delta \delta \lambda_1|)  \nonumber\\ &\hspace{1mm}= \mathbb{E}\left(\frac{|\delta \lambda_1^{\dagger}-\delta \lambda_1|}{\delta \lambda_1}\right)
\end{align}


In order to calculate the size of terms inside the $\bm{\Delta}$ matrices, Equation \ref{eqn:p_dist} is used, where $\mathcal{U}$ refers to the uniform distribution, the $i,j = 1,...,d$ indices refer to each term in the matrix, and $d = 1,...,20$ defines the dimensionality (size) of the matrices. Through this definition there are $d^2$ degrees of freedom inside the matrix at any time.

\begin{equation} \label{eqn:p_dist}
    \bm{\Delta}_{i,j} \sim p\cdot\mathcal{U}[0,1], 
\end{equation}

The value of $d$ changes because we are considering the effect of dimensionality on the quality of the linear step size, $\delta \lambda$. Moreover the scalar $p \in \{1/100,1/10,1,10\}$  defines the magnitude of the terms in the $\bm{\Delta}$ matrices. Thus $p$ in a practical sense (that is, in reference to a gradient-based optimisation algorithm) can be interpreted as the \textit{step size} of the algorithm.

From Figure \ref{fig:step_sizes} it can be seen that as the average magnitude of the elements inside the $\bm{\Delta}$ matrices increase, the absolute difference between the theoretical $\delta \lambda_1$ value, and those calculated via Equation \ref{eqn:pert_gen}, that is, $\delta \lambda^{\dagger}_1$, becomes larger. Even when the average step size takes value p = 1/100, there appears to be a bias in the magnitude of the error, which is made clear in the \textit{zoomed in} subplot of Figure \ref{fig:step_sizes_zoomed}. Moreover as the dimensionality of the $\bm{\Delta}$ matrices increase, the $|\Delta \delta \lambda_1|$ errors appear to increase slightly. Hence in summary as this experimental analysis of Equation \ref{eqn:pert_gen} seems to suggest, gradient based methods are potentially difficult to implement. In particular, it would appear that we would require $p< 1/100$ at a minimum, and that this value would need to continually decrease as dimensionality increases. For this reason, we opt to explore the viability of particle swarm optimsation as a means for optimisation since it is a non-gradient based approach, and gradient-based approaches seem to require very small step sizes for accurate gradients.

\subsection{Particle Swarm Optimisation}

Particle Swarm Optimisation (PSO) is a stochastic, evolutionary optimisation first proposed by Kenedy and Eberhart\cite{eberhart1995new}. The algorithm works by generating an array of candidate particles (the swarm) across an objective space. Within this space each particle searches for the global optimum through the sharing of information within the swarm in a classic exploration-exploitation trade-off. This is clarified in Equations \ref{eqn:PSO} and \ref{eqn:PSO2}. 

\begin{align} 
    v_{i}^{k+1} &= \omega v_{i}^{k} + c_1r_1(p_i - x_i^k) + c_2r_2(p_G - x_i^k) \label{eqn:PSO} 
\end{align}
\vspace{-10mm}
\begin{align}
        x_{i}^{k+1} &= x_{i}^{k} + \alpha v_{i}^{k+1} \label{eqn:PSO2}
\end{align}

where $\alpha$ denotes step size, $\omega$ controls the particle's inertia, $c_1$ and $c_2$ (known as the acceleration coefficients) are constants which control the degree of the \textit{exploration}-\textit{exploitation} trade-off, $p_i$ and $p_G$ are the local optima, and global optimum, for each, and across all particles respectively (that is each particle stores their own local optimum, but shares knowledge of the current global optimum), and $r_1, r_2 \sim \mathcal{U}(0,1)$. In this way, every particle is made aware of the current global optimum, and explores the objective space accordingly (as specified through the $c_1$ and $c_2$ constants). 

Empirical studies in PSO theory have shown that the correct choice of inertia weight is critical in ensuring convergent behaviour of the algorithm \cite{van2010convergence}. Prior investigations have suggested that the choice of inertia is driven by the acceleration coefficients through: $2\omega > (c_1 + c_2) - 2$. This region describes the set of all $c_1$, and $c_2$ values which guarantee convergent behaviour based on the spectral analysis of the matrix describing the PSO dynamics \cite{van2006study}. However this inequality should be only be taken as a rough guide since it was derived assuming the PSO system has one particle, and one dimension. Empirically however, Eberhart and Shi suggest using values of $\omega=0.7298$ and $c_1 = c_2 = 1.49618$ for \textit{good} convergent behaviour in general \cite{eberhart2000comparing}. 


\subsection{Random Embedding}

An aim of this paper is to explore the capability of random projections to reduce the underlying dimensionality of the generalised eigenvalue problem. In particular, we propose an extremely general parameterisation of the $\Delta \mathbf{M}$, and $\Delta \mathbf{K}$ matrices, and explore whether or not it is possible to solve this problem in a lower dimensional space. The lowest dimensional space in which the problem may be solved completely is known as the \textit{effective dimension}, and is denoted by $d_e$. In particular the following definition is used to strictly define the notion of $d_e$, where Definition 4.1 is based on Definition 1 of Wang et al. \cite{wang2016bayesian}.

\begin{definition}[Effective Dimension]
Suppose there exists a linear subspace $\mathcal{T} \subset \mathbb{R}^D$, where $\text{dim}(\mathcal{T}) = d_e < D$.
A function $f : \mathbb{R}^D \to \mathbb{R}$ is said to have \textbf{effective dimensionality} $d_e$, if $d_e$ is the smallest integer such that $\forall   x \in \mathcal{T}$ and $x^{\bot} \in \mathcal{T}^{\bot} \subset \mathbb{R}^D$, where $\mathcal{T} \small{\oplus} \mathcal{T}^{\bot} = \mathbb{R}^D$, we have $f(x + x^{\bot}) = f(x)$.
\end{definition}

A simple example to clarify this defintion for the reader may be seen if we define the following function: $f(x_1,x_2) = x_1^2 + x_2^2, $ $\text{where, } f : \mathbb{R}^{10} \to \mathbb{R}, \text{ and } x_1, x_2 \in \mathbb{R} $. In this example, although the original space of $f$ is assumed to be 10-dimensional, one may easily observe that it has an \textit{effective dimension} of 2 (that is, $d_e = 2$), since it clearly only makes use of 2 dimensions, of the 10 possible dimesions it has access to. The remaining 8 dimensions are \textit{ineffective dimensions}. Unfortunately, in practice we never really know the actual value of $d_e$, but we either know or assume from prior knowledge that our problem may have a lower dimensional representation. In other words, in practice we only ever know or use $d\in \mathbb{Z}$ dimensions in total, where $D \geq d \geq d_e$, and thus our random embedding generally occurs via random matrices with dimensionality $D\times d$.

Although the notion of effective dimensionality is developed, it does not explain how such random projections to lower dimensional spaces should occur. Ideally when projecting to a lower dimensional space we desire $||T(x_i - x_j) || \approx ||(x_i - x_j) ||$, where $T: \mathbb{R}^n \to \mathbb{R}^m$ is some linear operator. That is, we aim to reduce the dimensionality of a set of points in some Euclidean space, which approximately retains these pair-wise distances measures in this new, lower-dimensional subspace. A bound on the degree of \textit{distortion} that occurs to the original space when we project to a lower dimensional space is famously shown through the Johnson-Lindenstrauss (JL) Lemma \cite{johnson1984extensions}, stated in Lemma \ref{lem:rand_emb}.

{\lem[\textbf{Johnson-Lindenstrauss Lemma}]{For any $0 < \epsilon < 1$, and for any integer n, let k be such that \begin{align*} \label{lem:rand_emb}
 k \geq 4 \frac{1}{\epsilon^2/2 - \epsilon^3/3}\log(n)   
\end{align*}
Then for any set $X$ of $n$ points in $\mathbb{R}^d$, there exists a linear map $f: \mathbb{R}^d \to \mathbb{R}^k$ such that $\forall x_i \in X$,
\begin{align}
    (1-\epsilon)||x_i - x_j||^2 \leq ||f(x_i) - f(x_j)||^2 \leq (1+\epsilon)||x_i - x_j||^2
\end{align}
}}

In effect the JL Lemma tells us that the quality of the projection down to some dimension $k$, is a function of some allowable error tolerance, $\epsilon$, and the amount of points invovled in the projection $n$. In particular the relative Euclidean distances will be distorted by a factor of no more than $(1\pm \epsilon)$, where $\epsilon \in (0,1)$. Note that this makes no reference to the initial dimension of the points existed in before the projection occured. 

Although the JL Lemma is used commonly with large datasets, we are only working with indiviudal, possibly high-dimensional points, which are used as inputs into functions used in an objective function. Thus in the case of random projections of data points into functions we also must consider the effective dimension $d_e$. Theorem 2 of Wang et al.\cite{wang2016bayesian} implies that no matter the degree of distortion which may occur, there shall always exist a solution in this lower dimensional space. This theorem is restated here in order to self-contain the paper. 

\begin{theorem}[\textbf{Wang's Existance Theorem}]Assume we are given a function $f : \mathbb{R}^D \to \mathbb{R}$ with effective dimensionality  $d_e$ and a random matrix $A\in\mathbb{R}^{D\times d}$ with independent entries sampled according to $\mathcal{N}(0,1)$, where $d\geq d_e$. Then with probability 1, for any $x \in \mathbb{R}^D$, there exists a $y\in \mathbb{R}^d$ such that $f(x) = f(Ay)$\end{theorem}

That is, we should always be able to find some $y$ such that $f(x) = f(Ay)$, where $A \in \mathbb{R}^{D \times d}$ is some random matrix. And thus the distortion of the projection as predicted by the JL Lemma is not as important of a factor if we can ensure $d \geq d_e$, since with a good enough optimisation algorithm, if there exists a $x^{\star} \in \mathbb{R}^D$ which is optimal, then there exists a $y^{\star} \in \mathbb{R}^d$ such that $f(x^{\star}) = f(Ay^{\star})$. In practice however we may by chance select some $d < d_e$, and so in these cases it may become informative to use JL Lemma as a guide to assist in undesrtanding the degree of distortion which did indeed occur by projecting into this new subspace. 

In order to ensure that the point-wise distances in this new subspace abides by the JL Lemma we consider random matrices of the form: $A_{i,j} \sim \mathcal{N}(0,1/\sqrt{d})$. This is simply a scaled version of the random Gaussian matrices proposed by Wang et al. in Threom 2, but a Gaussian matrix of this form is known to better preserves distance properties in this new subspace \cite{johnson1984extensions}.



An issue which may occur when trying to use random embeddings for the purpose of optimisation is that the optimisation bounds which are defined in the larger $D$-dimensional space may be violated in the lower $d$-dimensional space. Thus, it is suggested to use a convex projection method to ensure that any variables $y \in \mathbb{R}^d$ fall the into bounding constraints defined by variables $x \in \mathbb{R}^D$. This idea is shown in Figure \ref{fig:embed_prop}.

\begin{figure}[H]
\centering
  \includegraphics[width=0.45\linewidth]{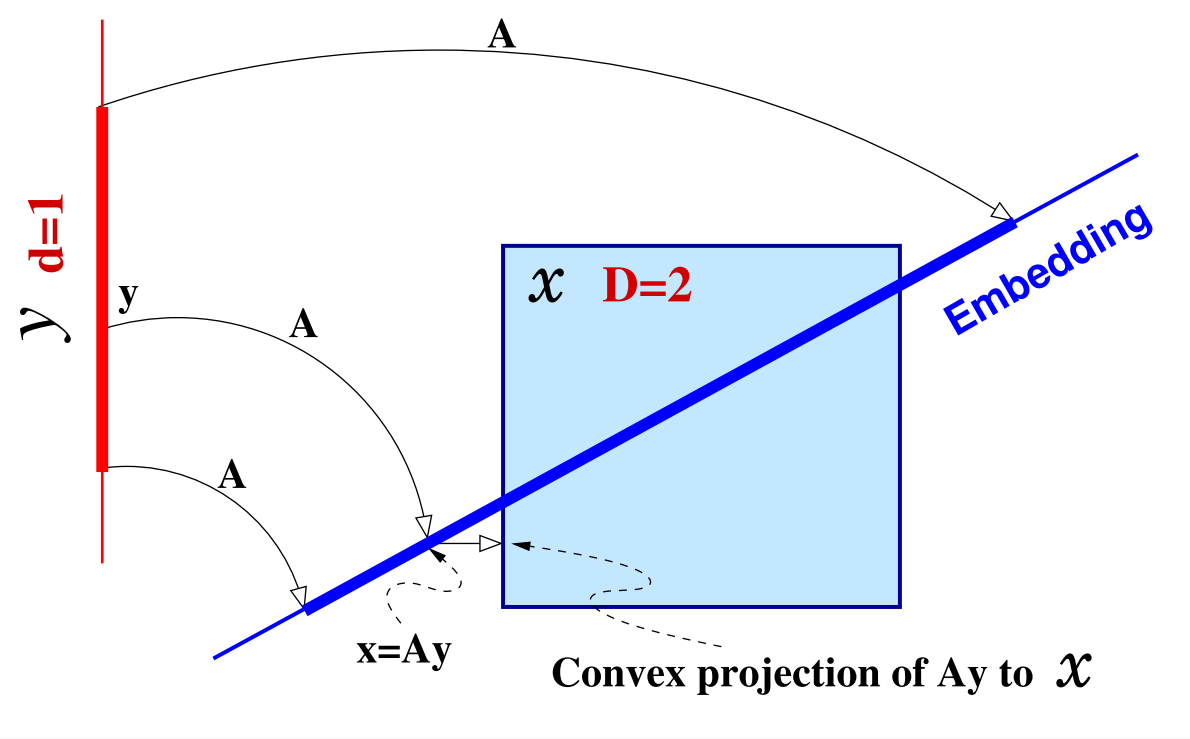}
  \caption{Need to perform a convex projection \cite{wang2016bayesian}.}
  \label{fig:embed_prop}
\end{figure}

Assuming that the feasible set for $\mathbf{x}$ is defined with box constraints, denoted by $\mathcal{X} \coloneqq [-c,c]^D$ where $c\in\mathbb{R}^{+}$, a simple way to ensure that $\mathbf{Ay}$ maps to the range defined by $\mathcal{X}$ may be achieved through a least squares method \cite{wang2016bayesian}. Mathematically, we define: $p_{\mathcal{X}}(\mathbf{Ay})= \mathrm{argmin}_{\mathbf{z}}||\mathbf{z-Ay}||_2^2$, where $\mathbf{Ay} \in \mathbb{R}^D$, $\mathbf{z} \in \mathcal{X}$, and $p_{\mathcal{X}} : \mathbb{R}^D \to \mathcal{X}$ denotes \textit{projection}. This least squares method effectively drops the perpendicular from points outside the bounding box, $\mathcal{X}$, which can be located aribtrarily in $\mathbb{R}^D$, towards the nearest point on the boundary of $\mathcal{X}$. This is made clear in Figure \ref{fig:embed_prop}. 

\section{Results}

\subsection{Ten-Dimensional Toy Problem}

 In this section the results of applying the combination of PSO and random embedding for inverse eigenvalue problems in structural engineering are explored. We begin by considering the 10 degree of freedom (DoF) system defined in Equations \ref{eqn:Mass_Sivan} and \ref{eqn:Stiff_Sivan}. The stiffness matrix is based on that of Sivan \& Ram \cite{sivan1996mass}, with the mass matrix being modified from a diagonal of ones to allow for a more complex scenario. 

\begin{align} \label{eqn:Mass_Sivan}
\mathbf{K} = 
\begin{pmatrix}
     200 & -10 & -20 & -5 & -5 & -10 & 0 & 0 & -50 & -50\\ -10 & 100 & 0 & 0 & 0 & 0 & -20 & -10 & -20 & -10\\ -20 & 0 & 300 & -40 & -30 & -60 & -10 & 0 & -20 & -10\\ -5 & 0 & -40 & 400 & -30 & -40 & -50 & -20 & -10 & -70\\ -5 & 0 & -30 & -30 & 150 & -10 & -5 & -5 & -20 & 0\\ -10 & 0 & -60 & -40 & -10 & 250 & 0 & 0 & 0 & -80\\ 0 & -20 & -10 & -50 & -5 & 0 & 120 & -5 & 0 & -10\\ 0 & -10 & 0 & -20 & -5 & 0 & -5 & 250 & 0 & -100\\ -50 & -20 & -20 & -10 & -20 & 0 & 0 & 0 & 350 & -40\\ -50 & -10 & -10 & -70 & 0 & -80 & -10 & -100 & -40 & 400 
\end{pmatrix}
\end{align}

\begin{align} \label{eqn:Stiff_Sivan}
     \mathbf{M} = \text{diag}(1,...,10)
\end{align}

The first two eigenvalues of the generalised eigenvalue problem defined through these particular mass and stiffness matrices are given in Equation \ref{eqn:eig_Sivan}.

\begin{align} \label{eqn:eig_Sivan}
    \bm{\Lambda} = \begin{pmatrix} 10.99, 19.12  \end{pmatrix}
\end{align}

Our aim for this problem will be to find some $\Delta \bm{M}$ and $\Delta \bm{K}$ matrices which will transform the system eigenvalues into those specified by Equation \ref{eqn:eig_Sivan_mod}. 

\begin{align} \label{eqn:eig_Sivan_mod}
    \bm{\Lambda}^{\star} = \begin{pmatrix} 2.00, 5.00 \end{pmatrix}
\end{align}

In particular we shall work to minimise the objective function defined previously in Equation \ref{eqn:obj_fun_init} in order to find the associated the $\bm{\Delta}^{\star}$ matrices. We assume without loss of generality that the $\bm{\Delta}$ matrices are upper triangular, and real valued since the underlying $\bm{M}$ and $\bm{K}$ matrices are Hermitian and so are by definition symmetric. That is, it suffices to perturb only the upper (or lower) triangular portion of the $\bm{M}$ and $\bm{K}$ matrices. This means that for each of $\Delta \bm{K} \in \mathbb{R}^D$and $ \Delta \bm{M}\in \mathbb{R}^D $, there are $D(D+1)/2$ free parameters. In this way the amount of free parameters grows on the order of $\mathcal{O}(D^2)$. We thus use random projections to reduce the effect of this quadratic complexity. One convergence was run for a random projection which reduced the overall dimensional size by an order of magnitude, and another was reducing it by a factor of a half.  Convergence results are shown in Figure \ref{fig:toy_data_conv}.

\begin{figure}[ht!]
\centering
\subfloat[The average convergence rates of a 110 dimensional space and a 10 dimensional space.]{
\hspace{-1.5cm}
\label{fig:toy_data_conv_a}
\includegraphics[width=0.55\linewidth]{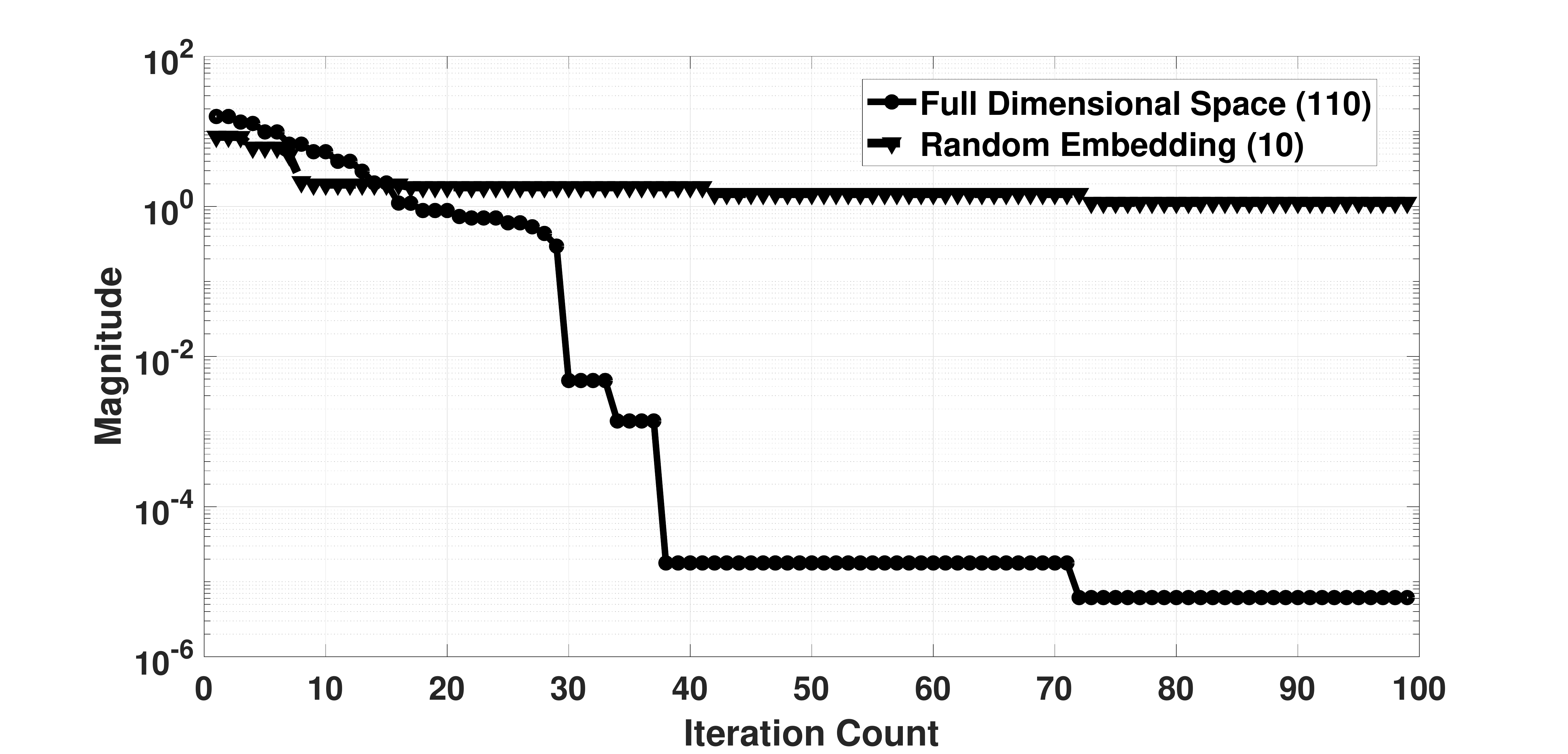}
}
\subfloat[The average convergence rates of a 110 dimensional space and a 50 dimensional space.]{
\hspace{-1cm}
\label{fig:toy_data_conv_b}
      \includegraphics[width=0.55\linewidth]{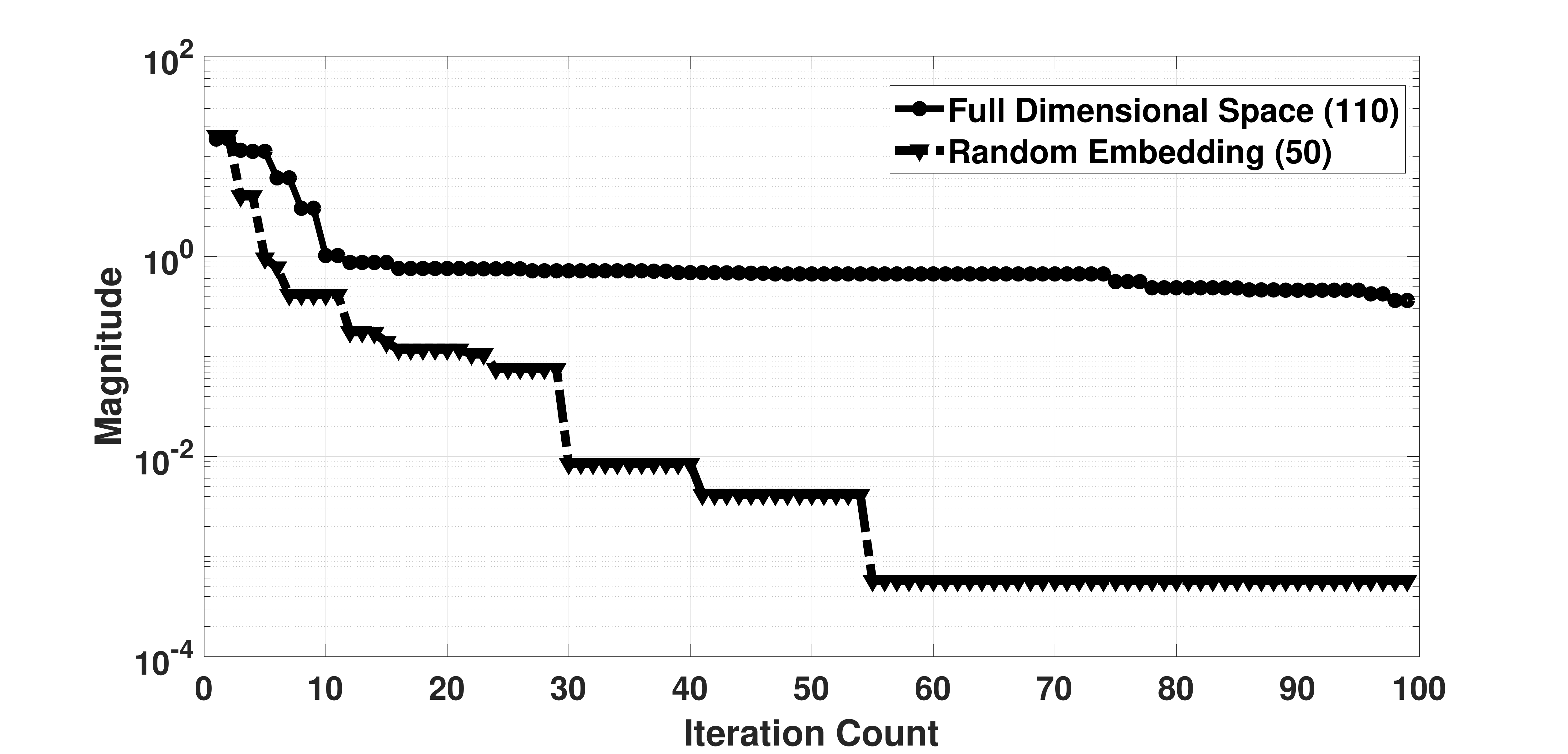}
}
  \caption{A comparison of the convergence rates between the full-dimensional optimisation problem, and the reduced dimension problem.}
  \label{fig:toy_data_conv}
\end{figure}

From Figure \ref{fig:toy_data_conv_a}, it would appear that the initial rate of decrease of the random embedding space is slightly faster than the full dimensional space. However, as Figure \ref{fig:toy_data_conv_a} then further suggests, although a low dimensional space may give a faster initial convergence rate, if the dimensionality reduction is too great, the optimisation routine can become plateau after a certain amount of iterations. This is seen in Figure \ref{fig:toy_data_conv_a} as the random embedding space remains at approximately $10^0$ after iteration 9. Thus it is clear that with this level of dimensionality reduction (for this particular problem), decreasing the dimensions by an order of magnitude seems to introduce a bias into the optimistion problem. It is conjectured that this is because the value of `10' lies well below the \textit{effective dimension} (the $d_e$ value) for this problem. This hypothesis is supported if we examine the same optimisation problem but instead reduce the exploration space to 50 dimensions as in the case of Figure \ref{fig:toy_data_conv_b}. We notice that not only do we have the faster initial decrease in optimisation rate, but also we converge to better values, faster. 

The reader may additionaly notice that in Figure \ref{fig:toy_data_conv_a} the 110 dimensional space reaches values as low as $10^{-6}$, but in Figure \ref{fig:toy_data_conv_b} it apparently stalls at $10^0$. However this is only due to the truncation of the plotting, since at has been hinted at approximately iteration 98 the full dimensional space starts to decrease its magnitude value again. But this only serves to demonstrate the notion within a 100 iteration limit the random projection has allowed to the optimisation to converge better values, significantly faster on average. Indeed, the full dimensional space will \textit{eventually} reach values as low as $10^{-6}$, however this toy problem suggests that on average it will not be as fast. Also note that since Figure \ref{fig:toy_data_conv_b} does not exhibit the same bias problems as in Figure \ref{fig:toy_data_conv_a}, it would appear that $10 \leq d_e \leq 50$. Hence although we have not been able to determine the \textit{effective dimension}, we can infer a range of existance for it. 

Also important to take note off is that for this problem the JL Lemma is not readily applicable since this 110 dimensional space is too low for the Lemma to take practical significance. In particular, if we set $n=110$ and assume error values of $\epsilon \in \{0.1,0.3,0.7,1.0\}$ we arrive at Table \ref{tab:JL_toy}. 

\begin{table}[h]
\centering
\begin{tabular}{ccccc}
\hline
Distortion Error (\%) & \textbf{10} & \textbf{30} & \textbf{70} & \textbf{100} \\ \hline
Dimension & 4029 & 523 & 143 & 112 \\ \hline
\end{tabular}
\caption{How the distortion error effects the corresponding dimension of the mapped subspace, for $n=110$ in accordance with the JL Lemma.}
\label{tab:JL_toy}
\end{table}

From Table \ref{tab:JL_toy} we see that for our toy problem, for an initial dimension of 110, the \textit{sufficient} dimension to guarantee no more than a 10\% error in the Euclidean distances between points in the new space is $4029 >> 110 > 50$. This number seems unreasonable because in lower dimensions the bounds predicted by the JL Lemma are not tight. That is these bounds serve give an idea of \textit{sufficiency}. Consequently by inspecting the mathematical equation of this bound, if the initial number of points is relatively low (as is the case for this toy example) we will not achieve a practically meaningful answer. However an important takeaway from the JL Lemma is that its formulation makes no reference to the initial dimension of projection, only the number of points considered, and so whether or not we began with $110$ points, or $110^{110}$ points, the \textit{sufficient} dimension to guarantee no more than a 10\% error between points after a random embedding is $110 < 4029 << 110^{110}$. Thus in this case it can be concluded that the initial $n=110$ value is too small for the JL Lemma to be directly useful.

\subsection{One-Dimensional Boeing-737 Finite Element Problem}

Here we shall explore how the PSO algorithm coupled with random embeddings may be exploited to assist in solving a truncated version of the generalised inverse eigenvalue problem for a 1D Boeing 737-300 (B737) Finite Element (FE) problem. In particular the model used to analyse the B737 plane is outlined in Figure \ref{fig:B737}. This model is based upon one found in Theory of Matrix Structural Analysis \cite{przemieniecki1985theory}.

\begin{figure}[ht!]
\centering
\subfloat[Frontal view of 1D Boeing 737-300 model showing the fuselage mass, and the wings.]{
\includegraphics[width=0.55\linewidth]{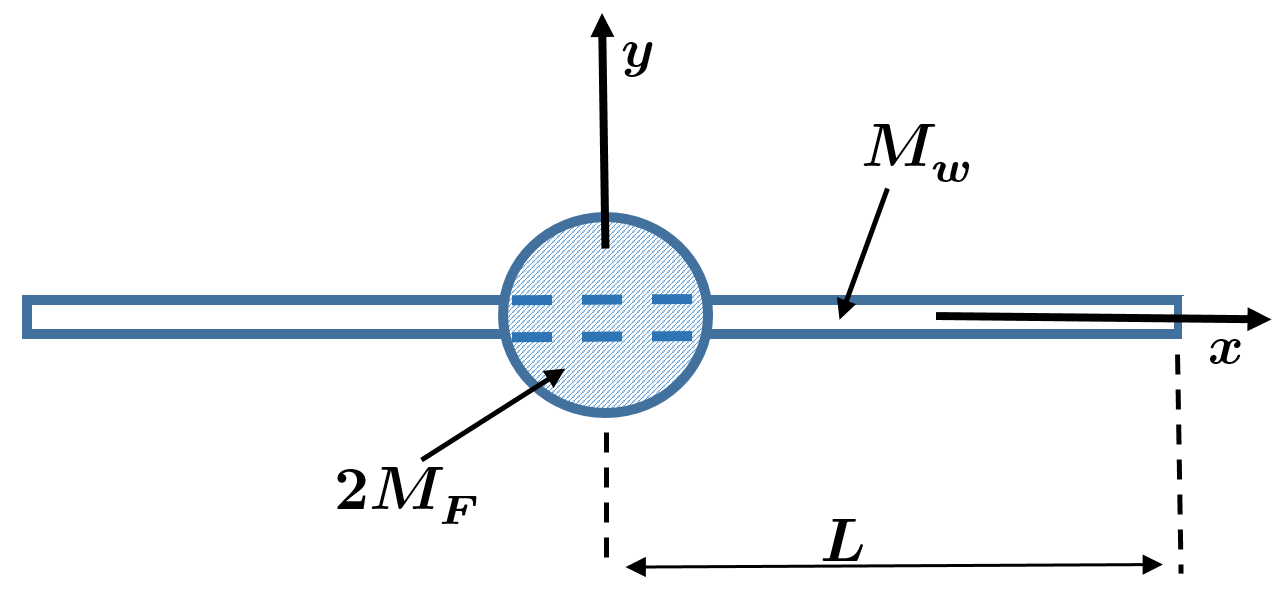}
    \label{fig:B737_a}
}
\\
\subfloat[Frontal view of the 1D point-mass FE discretisation of the Boeing 737-300 model.]{
      \includegraphics[width=0.45\linewidth]{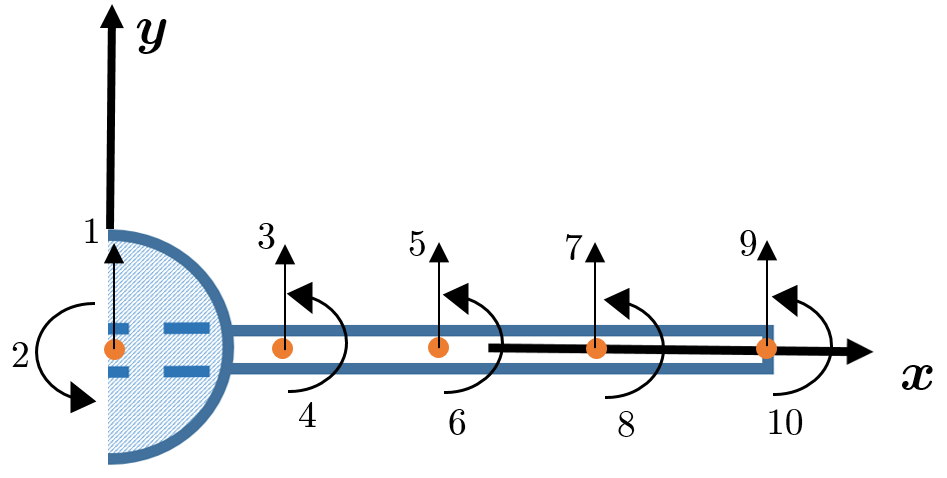}
      \label{fig:B737_b}
}
  \caption{Frontal views of the 1D Boeing 737-300 model which shall be used in the inverse eigenvalue problem. Images based upon Przemieniecki \cite{przemieniecki1985theory}.}
  \label{fig:B737}
\end{figure}

In Figure \ref{fig:B737} it is assumed that the total wing mass is uniformly distributed over the wing span of length $2L$, and its mass is $2M_w$. Moreover the total mass of the fuselage is $2M_F$. The wing elements are approximated as being finite element beam structures and have flexural stiffness given by $EI$, with the effects of shear deformations and rotary inertia neglected. Assuming only two FE nodes were used in this model, then the one element FE beam matrix for this problem is outlined in Equation \ref{eqn:FE_element}, where $R$ denotes the mass ratio between the fuselage and the wing, that is, $R = M_F/M_w$, and  $\mathbf{q} = \begin{bmatrix} w_1, \frac{\partial w_1}{\partial x},w_2, \frac{\partial w_2}{\partial x}  \end{bmatrix}^{\intercal}$. In order to estimate a value for $R$ the parameters for a B737 where obtained from literature, and summarised in \cref{tab:B737_general,tab:B737_fuselage,tab:B737_wing}. In order for the FE model to increase its accuracy more nodes must used, which involves constructing a large block matrix using the elements defined in Equation \ref{eqn:FE_element}.

\begin{align} \label{eqn:FE_element}
    \left( \frac{EI}{L^3} 
    \begin{bmatrix}
    12 & \text{symmetric} & & \\
    6L & 4L^2 & & \\ 
    -12 & -6L &  12 &\\
    6 & 2L^2 & -6L & 4L^2 
\end{bmatrix}  - \lambda M_w
  \begin{bmatrix}
  \frac{13}{35}+R & \text{symmetric} &  \\
  \frac{11}{210}L & \frac{L^2}{105} & & \\
       \frac{9}{70} & \frac{13}{240}L & \frac{13}{35} &  \\
    -\frac{13L}{420} & \frac{-L^2}{140} &-\frac{11L}{210} & \frac{L^2}{105} 
\end{bmatrix}  \right) \mathbf{q} = \mathbf{0}
\end{align}

As per the recommendation of Przemieniecki\cite{przemieniecki1985theory} the modal analysis of this structure may be separated into its symmetric and asymmetric counterparts, since this wing is symmetric around its fuselage centre. In order to enforece a symmetric condition, the second row and column of Equation \ref{eqn:FE_element} needs to be removed since it represents a degree of rotational freedom of the fuselage mass. For symmetry the fuselage mass is only allowed the move in a translational sense (up an down), and as such should not have any gradient (that is, not be able to rotate about some axis). The opposite is true in the case of anti-symmetry where instead the first row and column of the mass and stiffness matrices were removed, since in the case of asymmetry, the fuselage is allowed to rotate about an axis and thus have a defined gradient. Note that although doing this will not dramatically change the eigenvector response of the full system (if they are scaled properly), it can shift the eigenvalues appreciably. 

In order to assess the validity of this symmetric - asymmetric separation, Figures \ref{fig:sym_bending} and \ref{fig:asym_bending} representing the modal responses were constructed. Firstly, the mode shapes are consistent with those formulated by Przemieniecki \cite{przemieniecki1985theory}, and clearly there exists symmetry and asymmetry for the two shapes. Moreover we notice that the effect of the fuselage mass does have an appreciable albeit small effect on the symmetric modes, and no visible effect on the asymmetric modes, which agrees with Przemieniecki's analysis, and general intuition. This is because the removal of the first row and columns of the elemental beam matrix results also removes the $R$ variable. Notice also that in both cases (symmetric and asymmetric) there is an unconstrained mode, which mathematically exists due to the FE model having no fixed boundary conditions.

\begin{figure}[ht!]
\centering
\subfloat[The first two symmetric bending modes (modes 1 and 2).]{
\hspace{-1cm}
\includegraphics[width=0.55\linewidth]{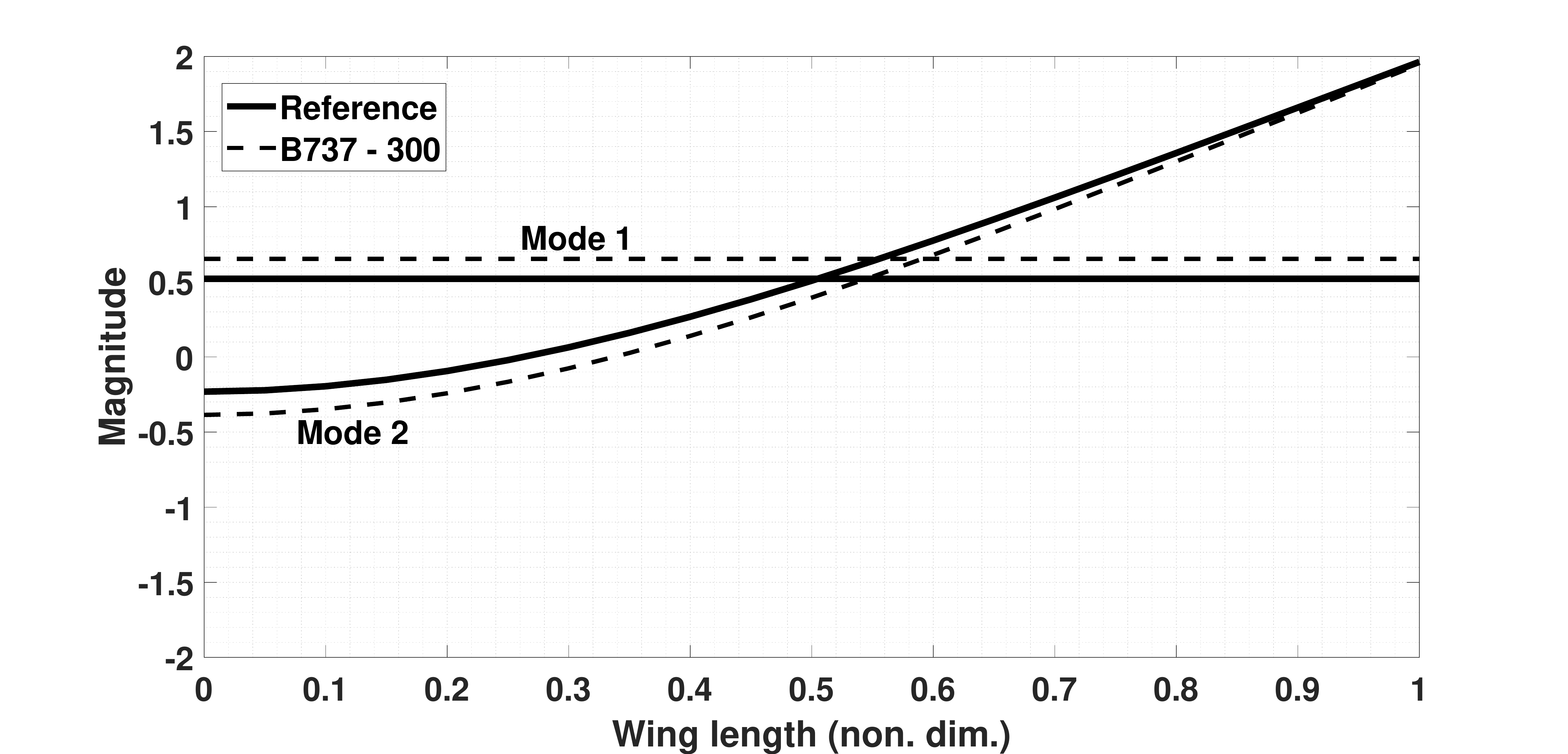}
  \label{fig:sym_bending_a}
}
\subfloat[The second two symmetric bending modes (modes 3 and 4).]{
\hspace{-0.8cm}
      \includegraphics[width=0.55\linewidth]{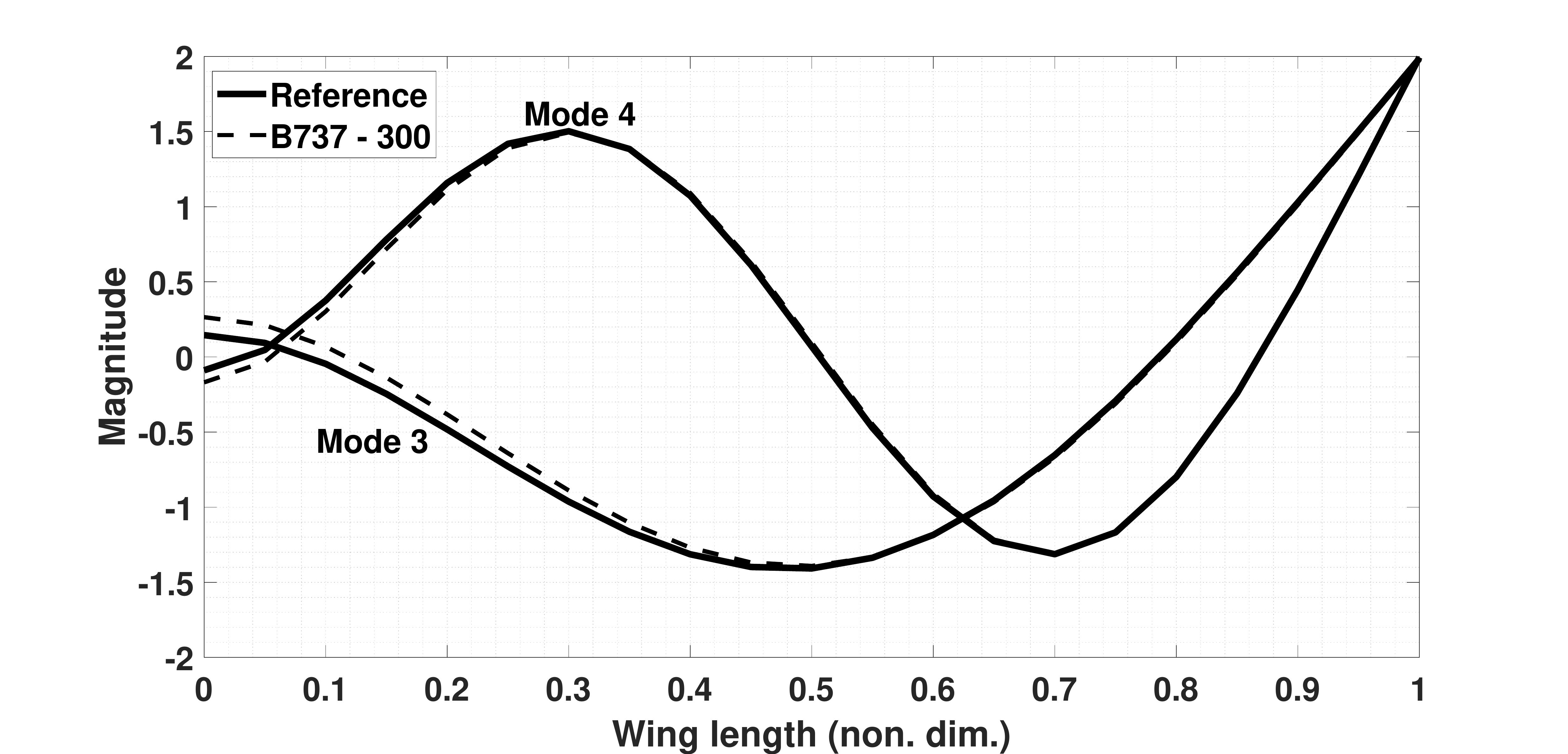}
    \label{fig:sym_bending_b}
}
  \caption{Symmetric bending modes for the 1D B737 FE model.}
  \label{fig:sym_bending}
\end{figure}

\begin{figure}[ht!]
\centering
\subfloat[The first two asymmetric bending modes (modes 1 and 2).]{
\hspace{-1cm}
  \label{fig:asym_bending_a}
\includegraphics[width=0.55\linewidth]{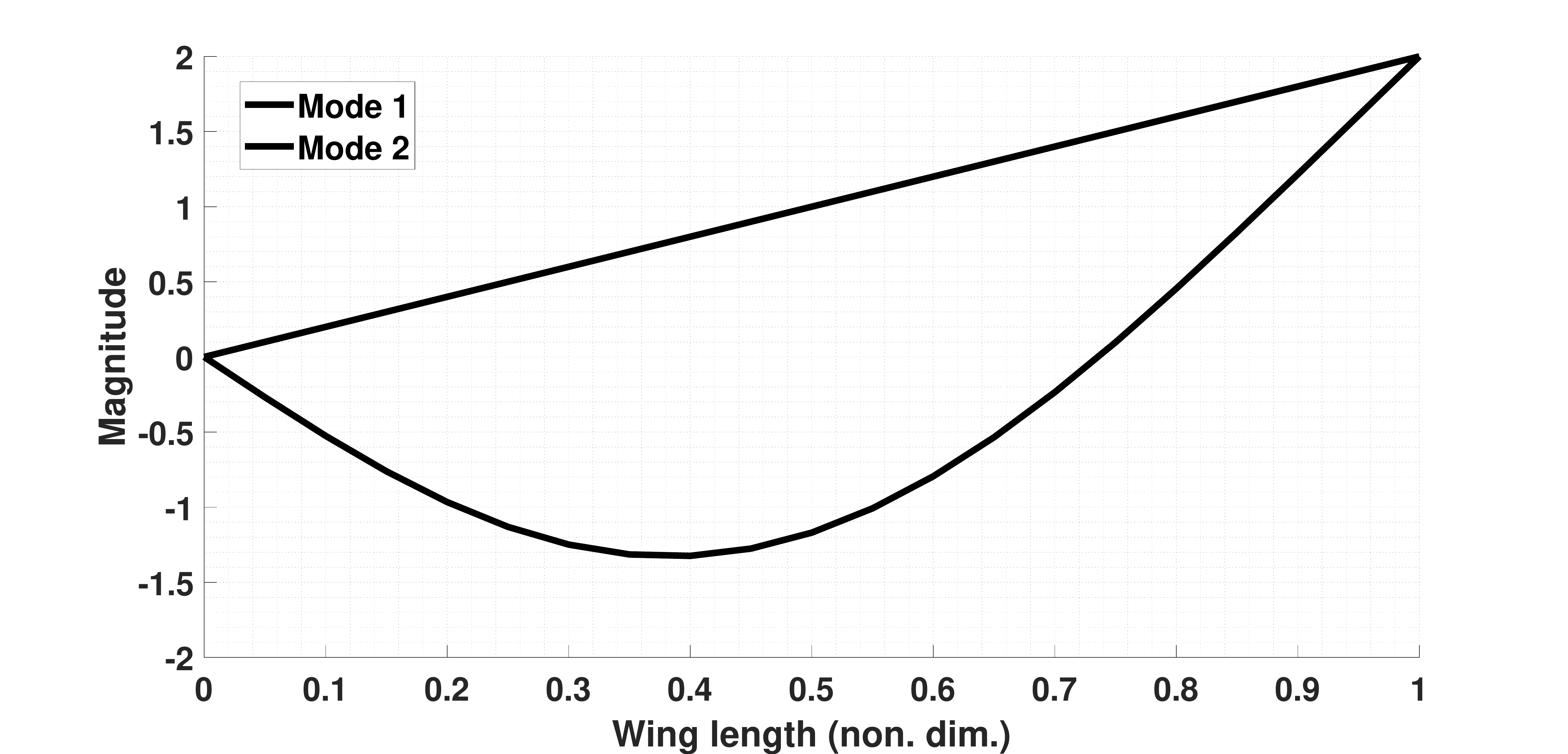}
}
\subfloat[The second two asymmetric bending modes (modes 3 and 4).]{
\hspace{-0.8cm}
      \includegraphics[width=0.55\linewidth]{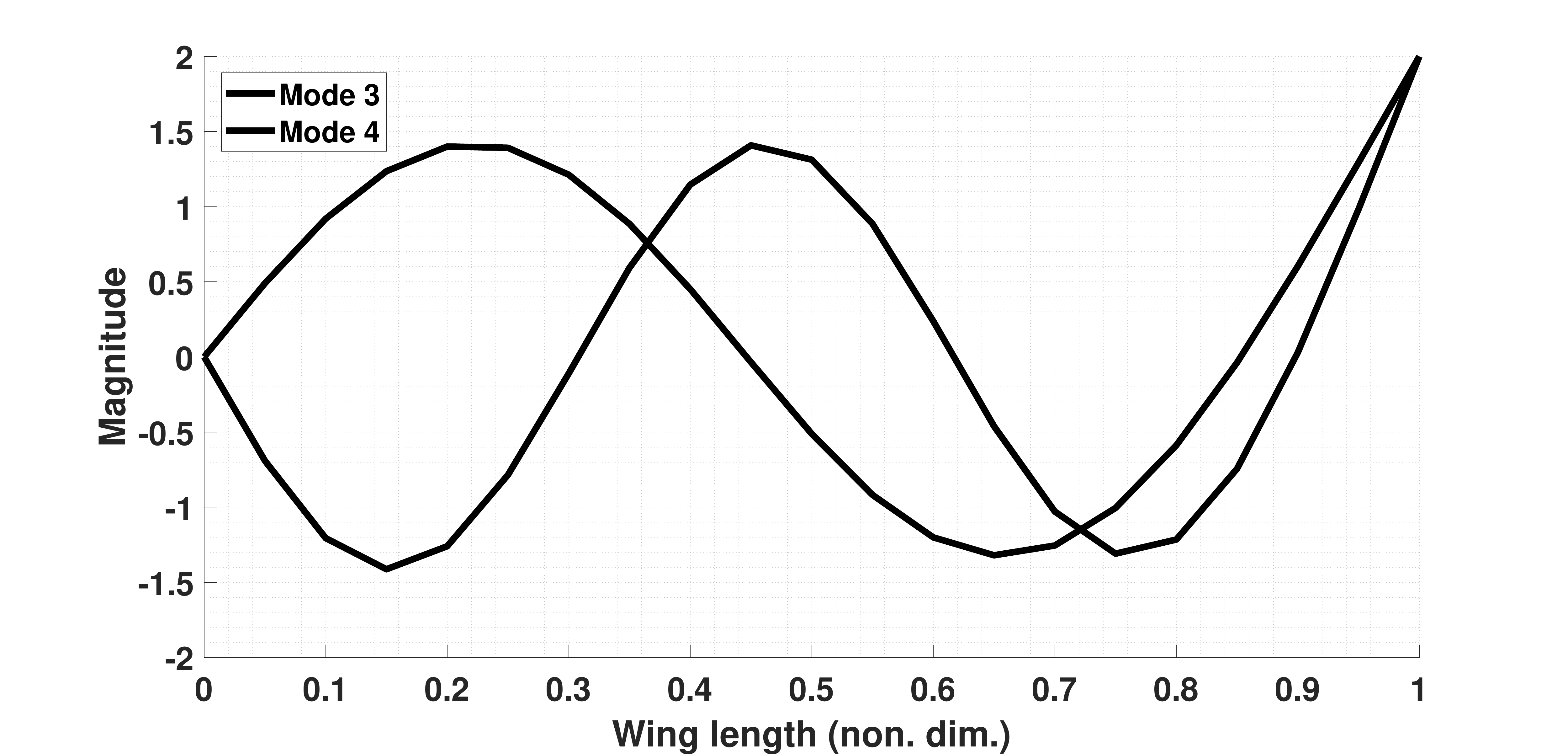}
  \label{fig:asym_bending_b}
}
  \caption{Asymmetric bending modes for the 1D B737 FE model.}
  \label{fig:asym_bending}
\end{figure}

In order to simplify analysis only the symmetric bending mode cases are considered, since as Figure \ref{fig:sym_bending_a} demonstrates, it takes into account the effect of the B737 model through the $R$ variable. As before with the toy example, without loss of generality we aim to formulate the $\bm{\Delta}$ matrices as upper triangular matrices, and perform the random embedding on these upper triangular matrices. However, different to the former case is that due to the reformulation of the problem as an FE model, it is now possible to arbitrarily grow the dimensionality of the problem by increasing the amount of elements of the FE model so that the optimisation problem can grow arbitrarily large, allowing for a more rigorous analysis of the potential usefulness and effects of random projections. 

The first three non-dimensional symmetric-mode frequency values of the system under investigation are given in Table \ref{tab:nondim_freq}. The frequency values given are non-dimensionalised according to $\lambda^2\sqrt{M_w L/(EI)}$, and $R=0$ refers to the base-line reference (if the aircraft purely consisted of beam elements and no fuselage mass), whereas $R=1.35$ refers to the B737 aircraft parameters, which are the values we will use in the optimisation procedure. As was mentioned earlier, the existence of the fuselage mass does indeed alter the eigenvalues in an appreciable manner. For the optimisation problem, we aim to alter the first three non-dimensional frequencies of the symmetric bending mode to become: $\omega = \begin{bmatrix} 2, 7, 22 \end{bmatrix}$. That is we would like the following mapping to occur between the eigenvalues, $\begin{bmatrix} 0,4.09,23.36 \end{bmatrix} \xmapsto{\bm{\Delta}} \begin{bmatrix} 2, 7, 22 \end{bmatrix}$. This is why only the first three eigenvalues are shown in Table \ref{tab:nondim_freq}. Note however any number of eigenvalues may be used, and that from a physical point of view, it does not necessarily make sense to be transforming the first eigenvalue from `0' to `2' since this changes the constraints of the system (as `0' represents rigid body motion). However, the emphasis of this paper is to explore PSO as applied to inverse structural eigenvalue problems in parallel with random projections, and so the objective functions were chosen arbitrarily. 

\begin{table}[H]
\centering
\begin{tabular}{ccc}
\hline
\textbf{Frequency Number} & \textbf{R = 0} & \textbf{R = 1.35} \\ \hline
1 & 0 & 0 \\
2 & 5.59 & 4.09 \\
3 & 30.23 & 23.36 \\ \hline
\end{tabular}
\caption{First three non-dimensional symmetric-mode frequency values, of the FE model aircraft, non-dimensionalised by $\lambda^2\sqrt{M_w L/(EI)}$, where $R=M_F/M_w$ is the fuselage-to-wing mass ratio.}
\label{tab:nondim_freq}
\end{table}

 The convergence behaviour for this optimisation problem is shown in Figure \ref{fig:conv_FE_Model}. In all three cases we note extremely similar behaviour as compared to the toy example. That is, convergent behaviour in the lower dimensional space is initially and consisitently much faster in the sense that (faster in the sense that with less iterations, the random embedding method tends to have a much lower objective function magnitude). Eventually however, the full dimensional space does \textit{tend} to approach similar values to the random embedding but this is to be expected, since the full system always perfectly describes the problem, and the problem at hand also does seem to possess a very low effective dimensionality, implying that the optimisation procedure may not need to actively explore all possible dimensions. That is, although the full dimensional space seems high, the particle swarm doesn't need to explore it fully to obtain a good solution. Regardless of the conjectured advantages of this particular problem, the random projection assists in the notion of \textit{faster} convergence behaviuor across all areas. 

\begin{figure}[ht!]
\centering
\subfloat[Convergence of a 5 element discretisation resulting in 132 free paramters.]{
\hspace{-1cm}
\includegraphics[width=0.55\linewidth]{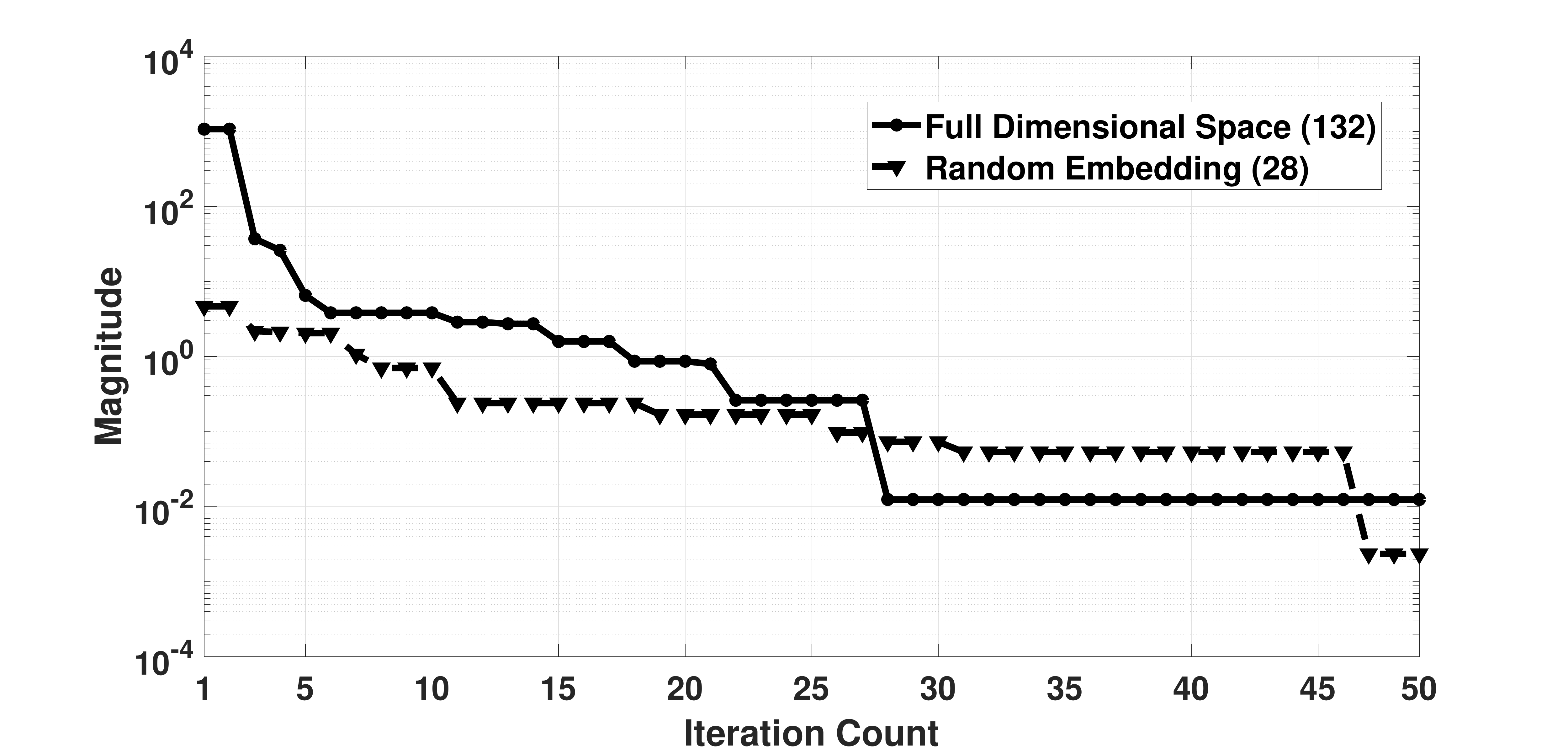}
}
\subfloat[Convergence of a 10 element discretisation resulting in 462 free paramters.]{
\hspace{-0.8cm}
      \includegraphics[width=0.55\linewidth]{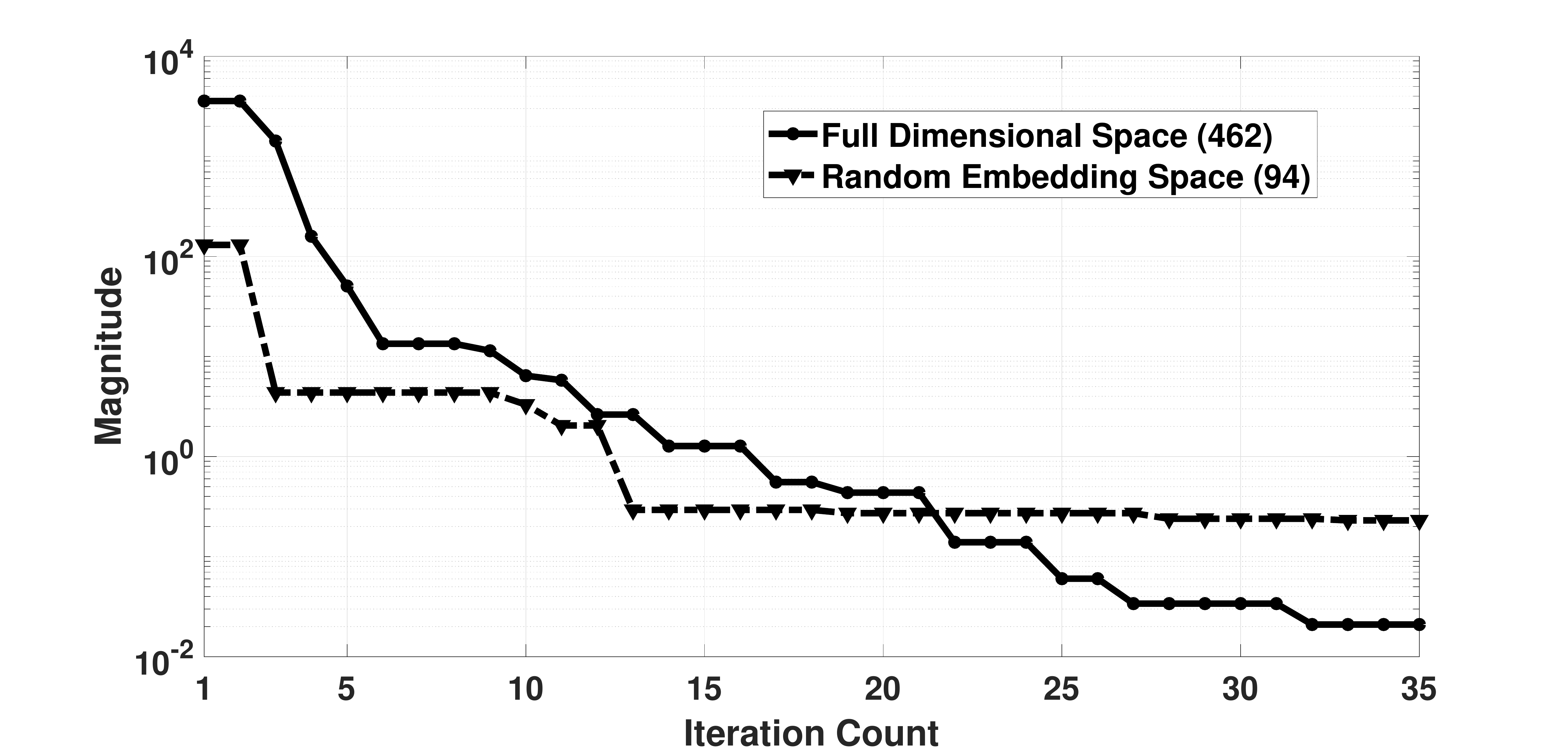}
}\\
\subfloat[Convergence of a 35 element discretisation resulting in 5112 free paramters.]{
\hspace{-0.8cm}
      \includegraphics[width=0.60\linewidth]{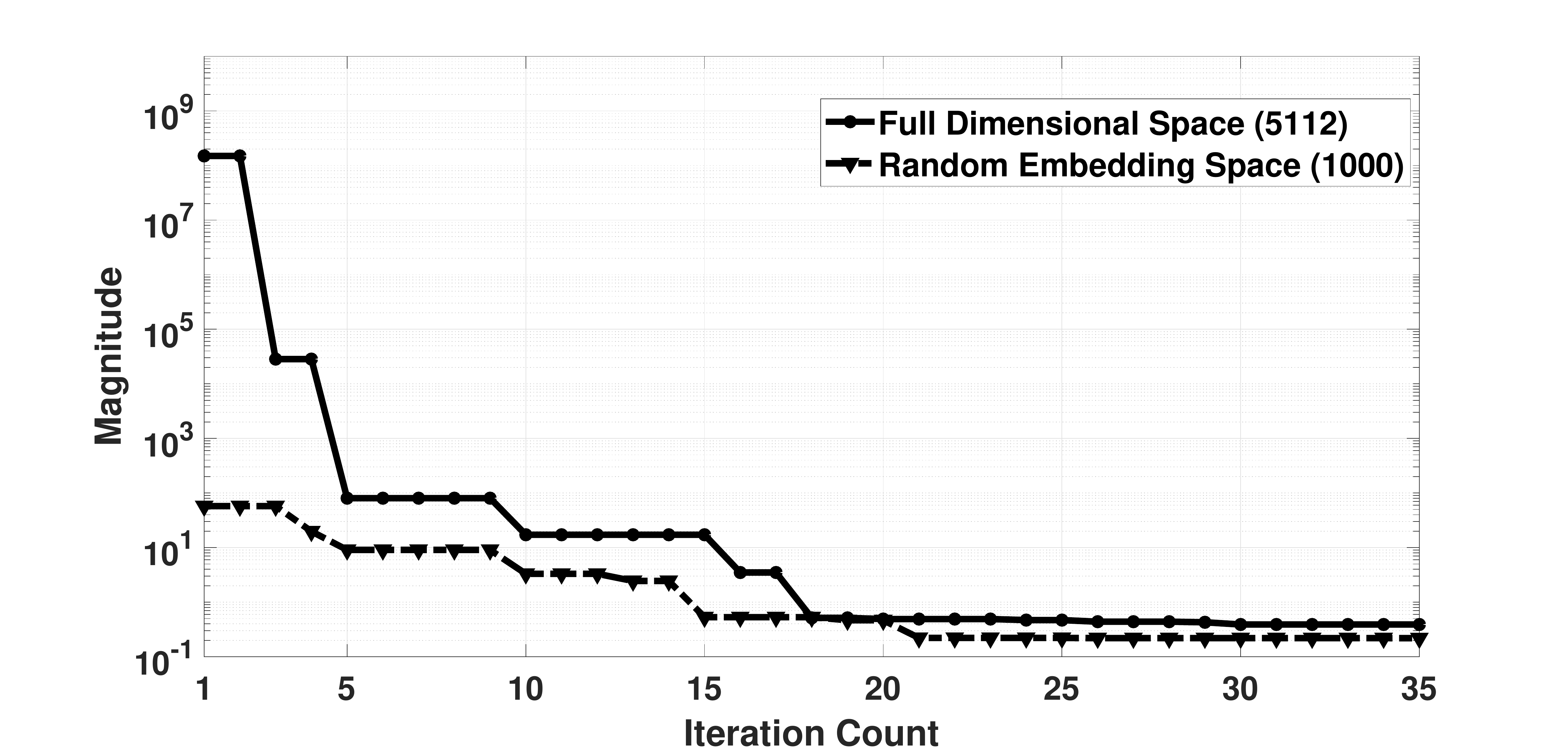}
}
  \caption{Convergence behaviour of the optimisation problem comparing the full dimensional, and dimensionally reduced spaces. In all cases a population of 500 particles were used, and a standard dimensional reduction of 80\% was used for consistency.}
  \label{fig:conv_FE_Model}
\end{figure}

In order to explore the capabilities of random embedding even further, it was applied to a case of a 100 element model for the aircraft. This resulted in a large search space of 40602 free parameters to explore for optimisation. It was proposed to reduce the dimensionality of the problem by 99.3\% resulting in a random embedding space of only 300 free parameters. In addition to this, the total amount of particles used in the swarm was reduced by a factor 2 (from 500 to 250). The result of this is shown in Figure \ref{fig:conv_FE_massive}.

\begin{figure}[ht!]
\centering
  \includegraphics[width=0.70\linewidth]{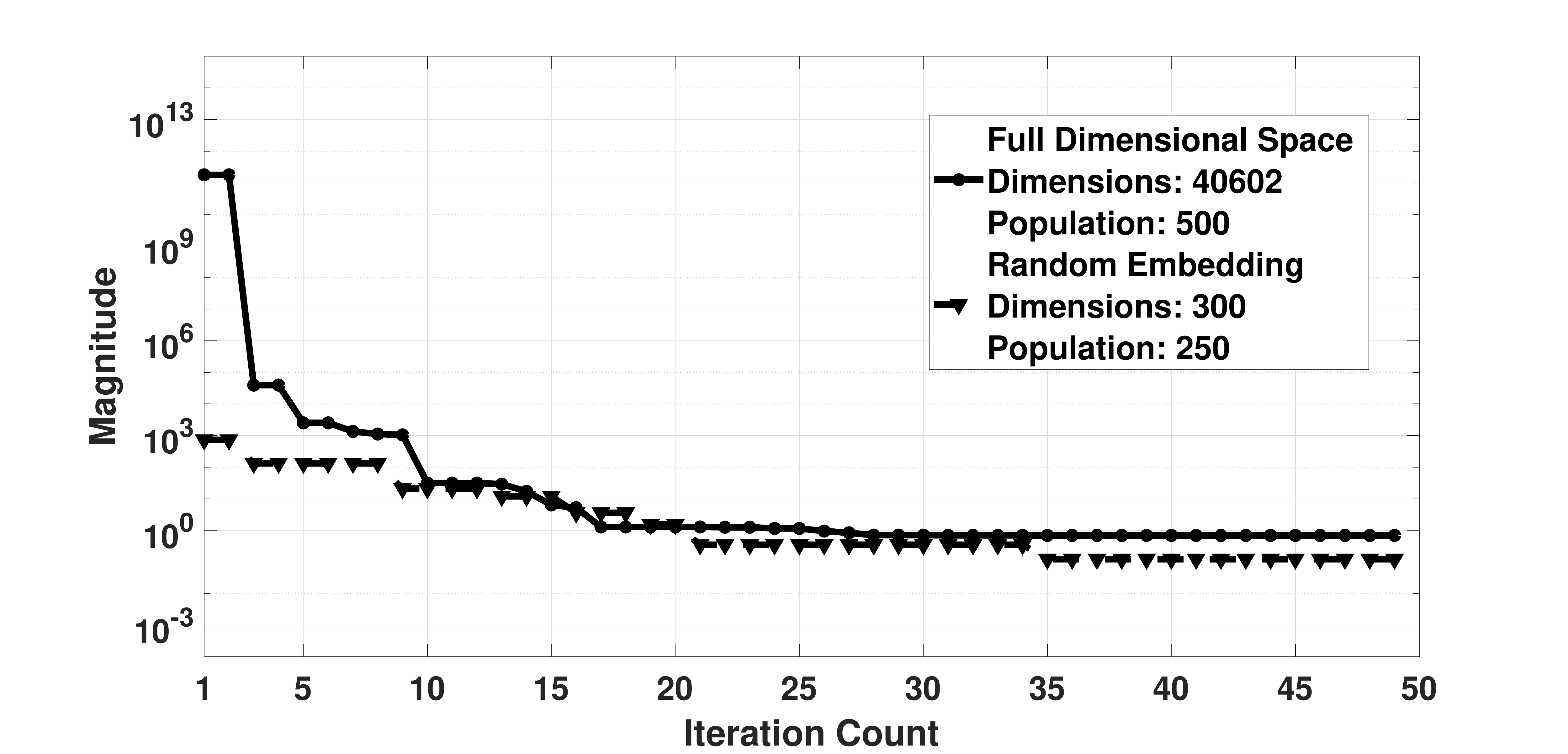}
  \caption{An extreme example showcasing the convergence of a 100 element discretisation resulting in 40602 free paramters. A dimension reduction of 99.3\% was used.}
  \label{fig:conv_FE_massive}
\end{figure}

From Figure \ref{fig:conv_FE_massive} once again extremely similar behaviour can be observed as in the previous problems. That is, the lower dimensional space is able to achieve much lower objective function magnitudes, a lot more rapidly. Moreover in this example, it was shown to be able to do this not only in less iterations, but also with less overall particles. In addition, the overall converged solution of this lower dimensional space is much better than the full dimensional solution which simply found it very difficult on average to converge to good values due to the enormous search space. The full dimensional solution could only converge on the order of $10^0$ on average, whereas the reduced dimension solution is able to converge to a value on the order of $10^{-2}$ on average. 

Thus as this paper has consistently demonstrated, through the use of random embedding we are able to significantly increase the speed and quality of convergence of a PSO optimiser, in terms of using less overall particles, coupled with less total iterations, ultimately leading to greater overall computational efficiency, in less total time. Note in this case we say `a' solution since the inverse eigenvalue problem with incomplete modal information is in a well known ill-conditioned problem and there does exist many locally optimal solutions. However the main purpose of this paper was not to explore the ability to achieve the global optima (which for the inverse eigenvalue problem, may not necessarily even be the best solution depending on context), but to analyse the applicability of using random projections alongside PSO in order to study the efficiency of an underlying optimisation procedure in the field of structural engineering.

A further point of discussion for the extremely high dimensional problem explored in Figure \ref{fig:conv_FE_massive} is to consider the level of distortion that has occured to the original surface when projecting down to a surface which has 99.3\% less overall dimensions. Table \ref{tab:JL_FE_massive} summarises the relationship for $n=40602$ for the JL Lemma. Here we note that for a 70\% average discrepancy between the pairwise Euclidean distancs of the points in the new lower dimensional space, 325 dimensions are \textit{sufficient}, which is comparable to what was used in Figure \ref{fig:conv_FE_massive}. Thus the geometry between adjacent points in this new subspace are most likely significantly different to what was in the original high dimensional space. Regardless however, it was nevertheless possible to converge to extremely good values suggesting that even though there may be a large geometric distortion, $d_e \leq 300$. And thus by Theorem 4.1 there exists a solution (or several), which we are able to find due to the strength of the black-box PSO algorithm. Note however that if we did not opt to reduce the dimensions by 99.3\%, but by a factor of $\approx 80\%$ we could still optimise with 9096 dimensions and achieve no more than 10\% error in the pairwise Euclidean distances between points. However the surface distortion issue does not appear to be a huge problem given that the problem has a low underlying effective dimension, and a good optimiser is used (as is the case of PSO). 

\begin{table}[h]
\centering
\begin{tabular}{cccccc}
\hline
Distortion Error (\%) & \textbf{10} & \textbf{30} & \textbf{50} & \textbf{70} & \textbf{100} \\ \hline
Dimension & 9096 & 1180 & 510 & 325 & 255 \\ \hline
\end{tabular}
\caption{How the distortion error effects the corresponding dimension of the mapped subspace, for $n=40602$ in accordance with the JL Lemma.}
\label{tab:JL_FE_massive}
\end{table}

Lastly it is important for the reader to note that the solutions obtained in this paper will be nonphysical. This is because the $\bm{\Delta}$ matrices are assumed to be full rank, upper-triangular matrices, without physical constraints applied to them (apart from symmetry being enforced via the upper-triangular nature of $\bm{\Delta}$). In order to enforce complete physicality of the solution it would be necessary to place constraints in the search space (either through equality and or inequality constraints). This idea has been explored partly by previous authors \cite{sivan1996mass,olsson2007inverse}, but it remains an open question in the case of truncated modal systems. Nevertheless, although it would be trivial to place constraints on the systems explored in this paper, it reamins that the purpose of this paper is to explore the viability of dimensionality reduction for structural vibration problems, of which the results appear to be extremely promising. The placement of constraints would not allow the justified exploration of spaces as high approximately $40000$ in the case of 1D FE model structures.

\section{Conclusion}
Random projection is a popular technique used to reduce the dimensionality of a problem. It has been demonstrated in this paper that by using random projections we were able to successfully perform optimisation in this lower dimensional space which resulted in much faster overall convergence, faster in the sense that on average less iterations were required to achieve much better results. This was demonstrated on an example 10-dimensional toy problem, as well as on a 1-D FE model of a Boeing 737-300 aircraft. Moreover the existence of a moderately small effective dimension was predicted to exist for generalised inverse eigenvalue problems which have Hermitian matrices. Moreover it was demonstrated experimentally that gradient-based approaches for performing optimisation for eigenvalue problems may necessitate prohibitively small step sizes, which tends to suggest that non-gradient, black-box optimisation methods may be preffered for these types of problems.

\pagebreak
\section*{\bibname}
\bibliographystyle{unsrt}
\bibliography{4241_che}

\pagebreak
\appendix

\section{Boeing 737-300 Data and Formulation}

An overview of the data used in modeling the 1D FE model B737 structure is presented in this section, as well as a the technical diagram used to extract some of its lengths. 

\begin{table}[H]
\centering
\begin{tabular}{ccc}
\hline
\textbf{Parameters}   & \textbf{Value}       & \textbf{Units}      \\ \hline
Cruise Velocity       & 725.43               & ft/s                \\
Cruise Altitude       & 30000                & ft                  \\
Air Density at Cruise & 8.91$\times 10^{-4}$ & slugs/$\text{ft}^3$ \\
Dynamic Pressure      & 234.44               & lb/$\text{ft}^2$    \\
Ultimate Load Factor  & 5.7                  & -                   \\
Design Gross Weight   & 109269.60            & lb                  \\ \hline
\end{tabular}
\caption{General Flight Parameters, available from \textit{Jane's all the World's Aircraft} \cite{taylor1976jane}.}
\label{tab:B737_general}
\end{table}

\begin{table}[H]
\centering
\begin{tabular}{ccc}
\hline
\textbf{Parameters}     & \textbf{Value} & \textbf{Units} \\ \hline
Length                  & 105.94         & ft             \\
Depth                   & 12.33          & ft             \\
Wet Area                & 4104.80        & $\text{ft}^2$  \\
Tail Length             & 15.89          & ft             \\
Cabin $\Delta$ Pressure & 8.00           & Pa             \\ \hline
\end{tabular}
\caption{Fuselage Parameters, available from \textit{Jane's all the World's Aircraft} \cite{taylor1976jane}.}
\label{tab:B737_fuselage}
\end{table}

\begin{table}[H]
\centering
\begin{tabular}{ccc}
\hline
\textbf{Parameters}      & \textbf{Value} & \textbf{Units} \\ \hline
Wet Area                 & 1133.90        & $\text{ft}^2$  \\
Weight of Fuel in Wing   & 35640.00       & lb             \\
Aspect Ratio             & 9.16           & -              \\
Wing Sweep at $25\%$ MAC & 25.00          & degrees        \\
Thickness-to-Chord ratio & 8.00           & -              \\ \hline
\end{tabular}
\caption{Wing Parameters, available from \textit{Jane's all the World's Aircraft} \cite{taylor1976jane}.}
\label{tab:B737_wing}
\end{table}

The equations used for estimating the fuselage and wing masses are available from Roskam \cite{raymer1999aircraft}, in particular the Equations used below refer to Equations (15.46), and (15.49) in Roskam. All terms of the below equatinos are are defined in this reference. 
\begin{align}
    &W_{\text{wing}} = 0.036S_w^{0.758}W_{fw}^{0.0035}\left(\frac{A}{\cos^2\Lambda}\right)^{0.6}q^{0.006}\lambda^{0.04}\left( \frac{100 t/c}{\cos\Lambda} \right)^{-0.3} (N_z W_{dg})^{0.49} \\
    &W_{\text{fuselage}}=0.052S_f^{1.086}(N_z W_{dg})^{0.177} L_t^{-0.051} (L/D)^{-0.072} q^{0.241} + W_{press} \\
    &W_{press} = 11.9 + (V_{pr}P_{\delta})
\end{align}

\newpage
\begin{figure}[ht!]
\centering
\subfloat[Top view of a Boeing 737-300]{
\includegraphics[width=0.60\linewidth]{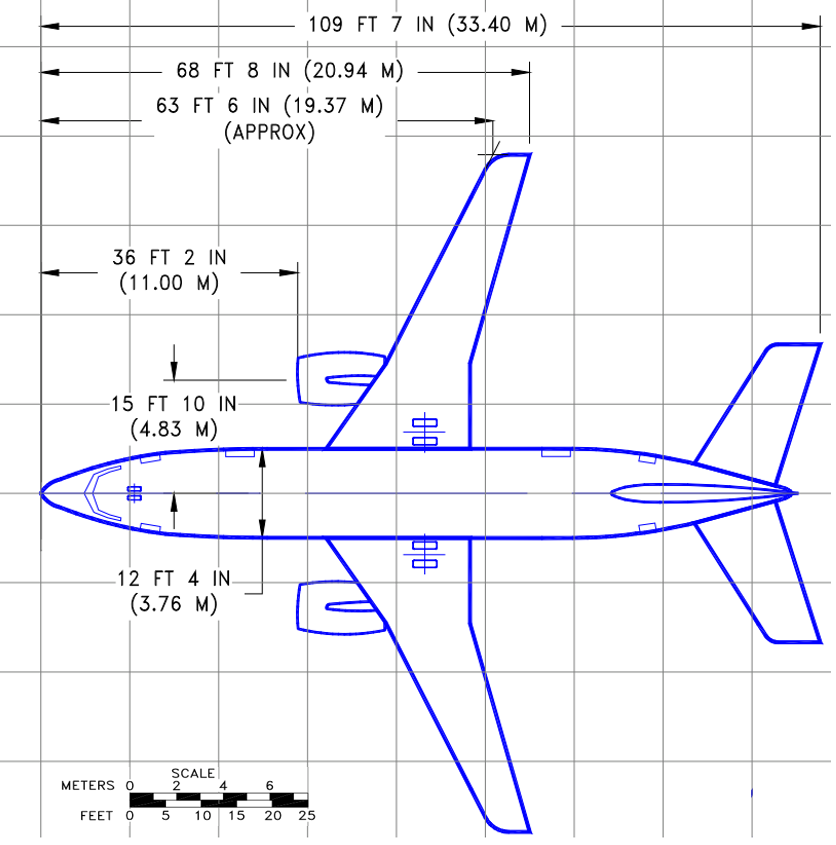}
}
\\
\subfloat[Side and front views of a Boeing 737-300]{
      \includegraphics[width=0.60\linewidth]{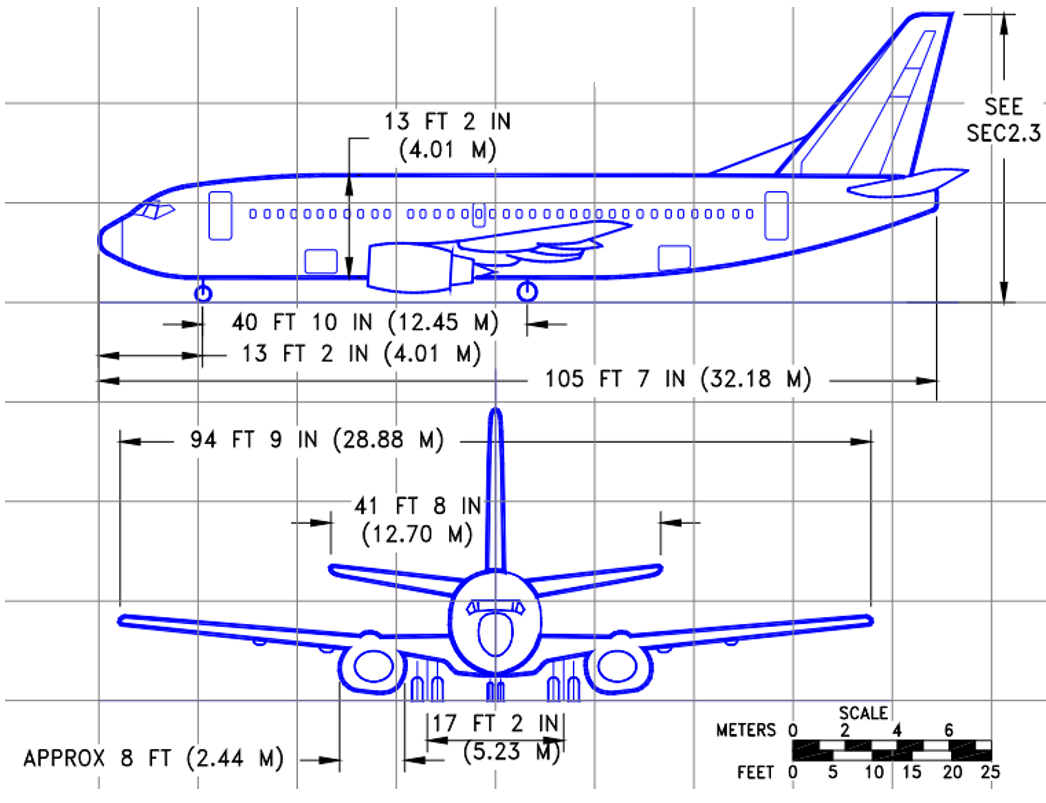}
}
  
  \caption{Design views of the Boeing 737 aircraft used to estimate certain lengths in \cref{tab:B737_general,tab:B737_fuselage,tab:B737_wing} \cite{B737Tech}. }
  \label{fig:B737_diagram}
\end{figure}

\end{document}